%% file: main.tex
\documentclass{article}

\usepackage{microtype}
\usepackage{graphicx}
\usepackage{subcaption}
\usepackage{booktabs} 

\usepackage{hyperref}

\usepackage{algpseudocode}
\usepackage[preprint]{icml2026}

\usepackage{amsmath}
\usepackage{amssymb}
\usepackage{mathtools}
\usepackage{amsthm}

\usepackage[capitalize,noabbrev]{cleveref}

\theoremstyle{plain}
\newtheorem{theorem}{Theorem}[section]

\newtheorem{lemma}[theorem]{Lemma}

\theoremstyle{definition}
\newtheorem{definition}[theorem]{Definition}
\newtheorem{assumption}[theorem]{Assumption}
\theoremstyle{remark}
\newtheorem{remark}[theorem]{Remark}

\usepackage[textsize=tiny]{todonotes}

\icmltitlerunning{Scalable Batch Correction for Cell Painting via  Batch-Dependent Kernels and Adaptive Sampling}

\begin{document}

\twocolumn[
  \icmltitle{Scalable Batch Correction for Cell Painting via 
  \\ Batch-Dependent Kernels and Adaptive Sampling}

  \icmlsetsymbol{equal}{*}

  \begin{icmlauthorlist}
    \icmlauthor{Aditya Narayan Ravi}{uiuc}
    \icmlauthor{Snehal Vadvalkar}{abbvie}
    \icmlauthor{Abhishek Pandey}{abbvie}
    \icmlauthor{Ilan Shomorony}{uiuc}
  \end{icmlauthorlist}

  \icmlaffiliation{uiuc}{Department of Electrical and Computer Engineering, University of Illinois Urbana-Champaign}
  \icmlaffiliation{abbvie}{AbbVie, USA}

  \icmlcorrespondingauthor{Aditya Narayan Ravi}{anravi2@illinois.edu}

  \icmlkeywords{Machine Learning, ICML}

  \vskip 0.3in
]



\printAffiliationsAndNotice{}  



\begin{abstract}
Cell Painting is a microscopy-based, high-content imaging assay that produces rich morphological profiles of cells and can support drug discovery by quantifying cellular responses to chemical perturbations. At scale, however, Cell Painting data is strongly affected by batch effects arising from differences in laboratories, instruments, and protocols, which can obscure biological signal. We present BALANS (Batch Alignment via Local Affinities and Subsampling), a scalable batch-correction method that aligns samples across batches by constructing a smoothed affinity matrix from pairwise distances. Given $n$ data points, BALANS builds a sparse affinity matrix $A \in \mathbb{R}^{n \times n}$ using two ideas. (i) For points $i$ and $j$, it sets a local scale using the distance from $i$ to its $k$-th nearest neighbor within the batch of $j$, then computes $A_{ij}$ via a Gaussian kernel calibrated by these batch-aware local scales. (ii) Rather than forming all $n^2$ entries, BALANS uses an adaptive sampling procedure that prioritizes rows with low cumulative neighbor coverage and retains only the strongest affinities per row, yielding a sparse but informative approximation of $A$. We prove that this sampling strategy is order-optimal in sample complexity and provides an approximation guarantee, and we show that BALANS runs in nearly linear time in $n$. Experiments on diverse real-world Cell Painting datasets and controlled large-scale synthetic benchmarks demonstrate that BALANS scales to large collections while improving runtime over native implementations of widely used batch-correction methods, without sacrificing correction quality.
\end{abstract}

\section{Introduction}\label{sec:introduction}

\input{intro}

\section{Key Concepts}
\label{sec:method}

\input{method}

\section{The BALANS Algorithm}
\label{sec:balances}
\input{balances}

\section{Experiments}
\label{sec:experiments}
\input{experiments}

\section{Concluding Remarks}
\label{sec:conclusion}
\input{conclusions}

\section{Impact Statement}
This paper presents work whose goal is to advance the field of machine learning. There are many potential societal consequences of our work, none of which we feel must be specifically highlighted here.

\medskip

\newpage

\bibliographystyle{unsrtnat}
\bibliography{references}
\clearpage
\onecolumn
\appendix
\section{Proof of Theorems}
\label{app:theory}
\input{theory}
\newpage
\section{Algorithms}
\label{app:algorithms}
\input{appendix_algo}
\newpage
\section{Run Time}
\label{app:run_times}
\input{run_time}
\newpage
\section{Ablations}
\label{app:ablations}
\input{ablations}
\newpage
\section{Experiment Details}
\label{app:dataset}

\input{dataset}
\newpage
\section{Evaluation Methods}
\label{app:methods}
\input{evaluation_method}

\newpage
\section{Metrics}
\label{app:metrics}

\input{metrics}

\newpage
\section{Real World Experiment Details and Full Results}
\label{app:realworld}

\input{realworld}
\newpage
\section{Synthetic Experiments and Results}
\label{app:synthetic}

\input{synthetic}
\section{Related Works}
\label{app:works}
\input{works}

\end{document}

%% file: intro.tex
Image-based profiling has emerged as a powerful tool for studying how cells respond to different treatments~\citep{regev2017human}.
Using high-throughput microscopy, researchers can detect changes in the shape and structure of cells 
caused by chemical or genetic perturbations~\citep{rohban2017systematic}. 
This allows the identification of promising compounds based on the cellular effects they produce, 
offering a way to accelerate drug discovery.
A widely used image-based profiling method is the Cell Painting assay \citep{bray2016cell, caicedo2017data}. 
Originally developed by \citet{bray2016cell}, Cell Painting uses six fluorescent dyes to label eight distinct cellular components, which are imaged across four or five channels. 
This approach provides high-content, single-cell resolution data 
that is cost-effective~\citep{cimini2023optimizing} and provides complementary information to \citep{cutiongco2020predicting, wawer2014toward} and protein-based assays \citep{dagher2023nelisa}.
Cell Painting has been successfully integrated with feature extraction techniques to generate numerical morphological profiles. 
A widely used tool is CellProfiler \citep{carpenter2006cellprofiler}, which extracts profiles based on engineered features such as texture, intensity, and shape descriptors from segmented cellular components. 
More recently, deep learning-based methods like DeepProfiler \citep{moshkov2024learning} have been introduced, using architectures such as EfficientNet \citep{tan2019efficientnet} to extract profiles. 

\textbf{Batch effects obscure true biological signal:}
Large-scale experiments often span multiple sources, (labs, protocols, microscopes) introducing batch-dependent artifacts that obscure true biological signals. 
The signal of interest is the morphological effect of a perturbation, independent of technical noise (e.g., plate or staining artifacts) or unrelated biological variation (e.g., cell density or nonspecific concentration effects). 
These confounding factors persist through standard pre-processing of the data.
A comprehensive analysis of batch effects in Cell Painting data and the performance of existing batch correction techniques was provided by~\citet{arevalo2024evaluating}.



\begin{figure*}[t]
\centering
\includegraphics[width=\textwidth]{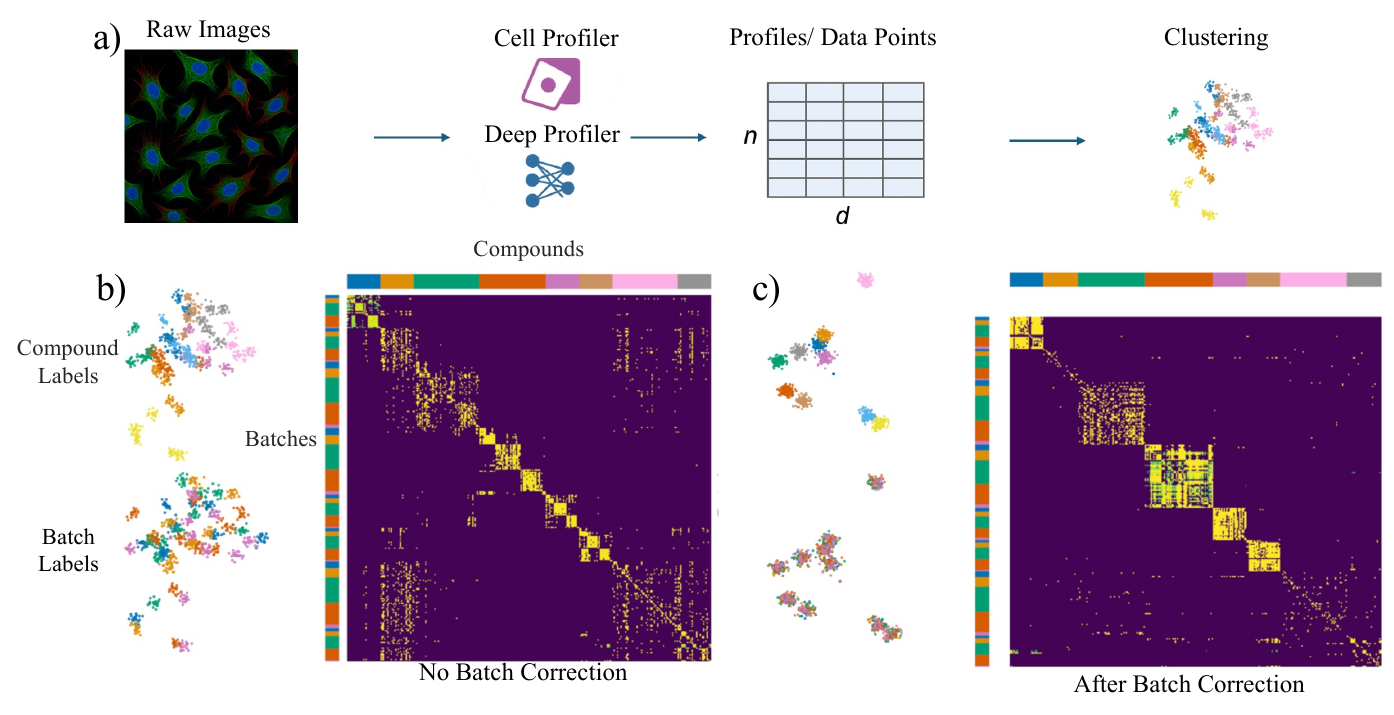}
\caption{
\textbf{(a)} 
Standard Cell Painting analysis pipeline based on extracting features for each of the $n$ cells.
\textbf{(b)} 
Before correction: when UMAP plot is colored by compound labels (top), we see multiple clusters corresponding to same compound.
When colored by batch, we see that the different clusters for the same compound correspond to different batches.
The affinity heatmap reveals block patterns aligned with batches. 
\textbf{(c)} After \textsc{BALANS}: Clusters corresponding to different compounds sharpen, batches mix, and affinities align with compound labels. 
Heatmap rows are ordered by batch (left color bar), columns by label (top color bar).
}
\label{fig:batch_correction_pipeline}
\end{figure*}


To illustrate batch effects, we analyze real Cell Painting data using a curated set of positive control compounds from the JUMP-CP consortium \citep{chandrasekaran2023jump}, which spans four datasets, 13 sources, and three microscopes. 
These compounds induce diverse morphological phenotypes and are thus ideal for visualizing batch effects. 
We used profiles extracted via CellProfiler 
The resulting affinity matrix is shown as a heatmap in Figure~\ref{fig:batch_correction_pipeline}(b), prior to any batch correction, revealing a block structure aligned with batch labels (colors on the left of heatmap).
The goal of batch correction is to perform a transformation on the cell feature vectors such that the resulting affinity matrix exhibits block structure that aligns with the compound labels, as 
illustrated in Figure~\ref{fig:batch_correction_pipeline}(c).


\textbf{Affinity-Based Denoising:} 
Denoising methods are widely used to enhance biological signal by smoothing over similar data points~\citep{van2018recovering, huang2018saver, li2018accurate}. 
They have shown success in mRNA and protein profiling \citep{tran2020benchmark}. 
For instance, \textsc{MAGIC} \citep{van2018recovering} constructs a 
carefully designed
affinity matrix to average each point with its neighbors. 
We take inspiration from this denoising technique and seek to build \emph{batch-aware} affinity matrices that can be used to remove batch effects.
One challenge that arises is that computing all pairwise affinities can be computationally expensive when we have many data points, and 
Cell Painting experiments 
often scale
to hundreds of thousands or even millions of single-cell profiles \citep{arevalo2024evaluating} across many different batches. 
Hence, we seek to design a \emph{batch-aware} affinity matrix that can be used for batch correction in a \emph{scalable} way,
allowing the efficient batch correction of large datasets.

\textbf{BALANS:} In this work, we introduce BALANS---Batch Alignment via Local Affinities and Subsampling---a novel, scalable 
batch correction method. 
Unlike traditional affinity-based approaches, BALANS incorporates batch label information into the construction of the smoothing affinity matrix. 
Scalability is attained by only computing a
\textit{sparse} version of the full affinity matrix, which allows each data point to be smoothed based on a small number of neighbors.

The first key idea behind BALANS is to make the notion of ``closeness'' depend on the batch. 
For any pair of data points $i$ and $j$ corresponding to unique profiles, BALANS adjusts their similarity by computing a local scale that depends on the batch of $j$. 
More precisely, if $b_j$ is the batch label associated with $j$, it computes the distance from $i$ to its $k$-th nearest neighbor within the batch of $j$, denoted $\sigma_{i,j}:= d_{i, (k)}^{b_j} $, and uses this as a scaling factor in a Gaussian kernel. 
The affinity between $i$ and $j$ is then defined as
\(
A_{ij} = \exp\left(-\|x_i - x_j\|^2/{\sigma_{i,j}^2}\right).
\)
This scaling has two effects. 
It increases the relative influence of compatible points from distant or less tightly clustered batches and suppresses misleading affinities from nearer or more tightly clustered batches. 
In doing so, BALANS selectively identifies cross-batch neighbors that are truly similar despite batch-induced distortions.

The second idea in BALANS is to avoid computing the full $n \times n$ affinity matrix, which is infeasible at scale. 
Instead, BALANS selects a subset $S \subseteq [n]$ with $|S| \ll n$ data points and computes their affinity rows. 
Rather than sampling uniformly, BALANS uses a coverage-based adaptive strategy: at each step $t$, the next point $i_t$ is selected with probability proportional to the inverse of its cumulative affinity to the previously sampled points. 
This prioritizes data points whose neighborhoods have not yet been well represented in earlier samples.
Each computed affinity row is sparsified using an elbow detection method, retaining only the most informative entries. 
The resulting matrix $A_S \in \mathbb{R}^{|S| \times n}$ captures a partial view of the full affinity matrix $A \in \mathbb{R}^{n \times n}$. 
The remaining entries are inferred via a low-rank approximation based on $A_S$ \citep{williams2000using}.
BALANS is summarized in Algorithm~\ref{alg:balances}.

We theoretically prove that this sampling strategy guarantees a similar number of rows from each cluster with high probability, using only $|S| = O(K \log K)$ samples. 
This ensures coverage of the underlying structure, enabling a low rank approximation of $A$. 
 We show that the reconstruction error $\|\widehat{A} - A\|_{\rm op}$ is bounded with high probability, where $\widehat{A}$ is constructed using the sampled rows.
BALANS runs in \emph{almost-linear time} with respect to the number of data points, making it scalable to large datasets. 
Importantly, the full affinity matrix $A$ is never explicitly constructed. 
Instead, BALANS applies a sparse, smoothing operator defined based on $A_S$ to the data matrix $X$.

\textbf{Experiments on Large Datasets:} 
We evaluate BALANS on real-world datasets from the JUMP Cell Painting Consortium \citep{chandrasekaran2023jump} and BBBC \citep{ljosa2012annotated}, with different scenarios using either CellProfiler or DeepProfiler features (Table~\ref{tab:table2}). 
Across eight benchmark scenarios, BALANS performs on par with state-of-the-art batch correction methods, and in the largest-scale, highest-diversity setting (Table~\ref{tab:table2}), it achieves over a 30\% improvement in average evaluation scores relative to the uncorrected baseline. 
We also report runtime comparisons, showing that BALANS is fast and competitive with existing methods.
In addition to accuracy, BALANS is highly efficient. Despite being implemented in Python, it runs faster than many optimized native implementations. 
On synthetic benchmarks designed to mimic real batch structure, while maintaining high correction quality, it scales to datasets with up to 5 million points in under an hour.
These results highlight BALANS as a scalable and effective solution for batch correction in high-content imaging.

\textbf{Related Works:} 
Cell Painting has been extensively explored using deep learning-based feature representations beyond CellProfiler and DeepProfiler \citep{doron2023unbiased, wong2023deep, kraus2024masked}. 
Batch correction has been extensively studied across a range of biological data modalities \citep{ando2017improving, vcuklina2021diagnostics, yu2023correcting, tran2020benchmark}. 
Local bandwidth estimation \citep{zelnik2004self, herrmann1997local, knutsson1994local} and low-rank approximation \citep{williams2000using, drineas2005nystrom, kumar2012sampling,musco2017recursive} also have well-established foundations. 
A comprehensive review of related works is provided in Appendix~\ref{app:works}.

\textbf{Comparison to BBKNN:}
 BBKNN \citep{Polanski2020BBKNN} is a tool that constructs a batch-aware KNN graph, and is related to BALANS. But its functionality differs fundamentally from BALANS. 
 BBKNN produces only a graph structure and does not generate corrected profiles or embeddings. 
 Although it defines an affinity matrix, BBKNN does not provide a way to efficiently apply this matrix to the data. 
 BALANS, in contrast, is explicitly designed to output corrected profiles without ever forming the full matrix.  Because BBKNN cannot produce corrected profiles, comparisons are limited to graph-based metrics.
A more detailed discussion, along with quantitative results, is provided in the Appendix~\ref{app:realworld}.

%% file: method.tex
In Cell Painting assays, cells are subjected to various perturbations (e.g., chemical or genetic), imaged using fluorescence microscopy, and then processed through a feature extraction pipeline to obtain profiles representing cellular morphology. 
Consider $n$ data points $x_1,\dots,x_n\in\mathbb{R}^d$ obtained in this way.  
Each point is annotated with a \emph{known}  batch label
$b_i\in\{1,\dots,B\}, \; i=1,\dots,n$
reflecting differences such as laboratory, protocols, or microscopes used. 
In addition, every point belongs to an \emph{unknown} biological cluster. 
A true biological cluster is a group of points that share a common underlying phenotype or similar phenotypes, regardless of their treatment labels or batch assignments. 
Each point thus has a (unknown) cluster label  $c_i\in\{\mathcal C_1, \mathcal C_2\dots,\mathcal C_K\}, \; i=1,\dots,n$.
Our goal is to identify, for each data point, a small set of neighbors that reflect true biological similarity with the datapoint. 
Using its neighbors we are to construct a smoothing operator that corrects it.

To identify neighbors for smoothing, one relies on a similarity measure between samples. 
These similarities are commonly encoded in an affinity matrix, which can be constructed in various ways (kernel functions \citep{hofmann2008kernel}, nearest-neighbor graphs \citep{abbasifard2014survey}, or learned distance metrics \citep{kulis2013metric}). 
However, when affinities are computed directly from the data without correction, they typically reflect both cluster membership and external sources of variation, including batch dependent effects and measurement noise.
To account for batch-dependent distortions in pairwise similarities, we define a \textbf{batch-dependent local scale} for each point, capturing how distant each batch appears from its perspective. 
These scales are then used to compute affinities via a local-scale Gaussian Kernel \citep{zelnik2004self}, where the similarity between two points depends on this scale.


\subsection{Batch-Dependent Local Scales}\label{ssec:bclb}

Let $d_{ij} = \lVert x_i - x_j \rVert_2$ denote the Euclidean distance between data points $x_i$ and $x_j$. 
A standard method for defining affinities from pairwise distances is using the Gaussian kernel $A_{ij} = \exp\left(-d_{ij}^2/\sigma^2\right)$,
where $\sigma > 0$ is a global scale controlling the scale of locality. 
This kernel assigns higher affinity to closer points and is widely used in spectral clustering \citep{zelnik2004self}. 
However, the assumption of a fixed global scale $\sigma$ is problematic when different batches may exhibit distinct density patterns, noise levels or distances. 
Thus, this can lead to poor affinity estimates.

This issue can be observed directly on real data. 
Consider Figure~\ref{fig:batch_correction_pipeline}(b), which shows a correlation matrix constructed as above, between data points that correspond to positive control compounds.
Although these compounds are known to induce strong and diverse phenotypes, the resulting structure exhibits strong clustering within batches, which is undesirable.
Therefore to address this issue, we define a batch-dependent local scale and the affinity associated with it.

\begin{definition}[Batch-dependent local scale; Affinity]\label{def:bclb}
Let $d_{ij} = \lVert x_i - x_j \rVert_2$ denote the Euclidean distance, and let $b_j$ be the batch of $x_j$. 
Let 
$
\sigma_{ij}^2 := \left(d_{i, (k)}^{b_j}\right)^2,
$
where $ d_{i, (k)}^{b_j} $ is the distance from $x_i$ to its $k$-th nearest neighbor within batch $b_j$. 
Then the batch-dependent affinity is 
$
A_{ij} := \exp\left(-d_{ij}^2/\sigma_{ij}^2\right).
$
\end{definition}

To illustrate its effect on real data, we apply the batch-dependent affinity to data points corresponding to positive controls and construct a correlation matrix. 
As shown in Figure~\ref{fig:batch_correction_pipeline}(c), the resulting heatmap reveals improved structure, with clearer cross-batch alignment at a compound level.
Computing batch-dependent affinities by adapting local scales across batches recovers true neighbors under batch effects, but is computationally expensive. It requires $O(n^2)$ distances and repeated $k$-NN searches to compute $\sigma_{ij}$ for all pairs.  
To ensure scalability, we estimate the affinity matrix by sampling a subset of rows, each of which provides access to a full neighborhood. 
From each row we obtain the batch-dependent local scale for that point.  
This form of row-based sampling is also known as landmark sampling \citep{kumar2012sampling}. 
Unlike randomly sampling entries of $A$, it preserves the information needed to compute batch-dependent scales.  
In the next section, we show that accurate approximation of $A$ is possible using only a small subset of rows.


\subsection{Low Rank Matrix Approximation via Adaptive Landmark Sampling}\label{ssec:lowrank}

Before we proceed, we first require a suitable model on the structure of the affinity matrix.
Going back to our experiment on positive controls, we note that applying the batch-dependent scale reveals a block structure in the affinity matrix (Figure~\ref{fig:batch_correction_pipeline}(c)), reflecting a block diagonal matrix overlaid with noise. 
This motivates a natural decomposition into a block diagonal matrix component plus noise.
To formalize this, we model the observed affinity matrix $\widetilde A \in \mathbb{R}^{n \times n}$ as
\(
\widetilde A = A_{0} + E,
\)
where $A_{0}$ is a block diagonal matrix encoding clustering structure, and $E$ denotes a residual noise independent of $A_0$. 
We assume \(A_{0}\) is symmetric and low-rank, with the following structure.

\begin{assumption}[Block-structured affinities]\label{assump:block}
There exists a partition of the index set $[n]$ into true biological clusters $\mathcal{C}_1, \dots, \mathcal{C}_K$, where each cluster $\mathcal{C}_k$ has size $n_k$. 
The low-rank affinity matrix $A_{0} \in \mathbb{R}^{n \times n}$ then takes the block form
$
A_{0} := \mathrm{diag}(A^{(1)}, A^{(2)}, \dots, A^{(K)}),
$
where each diagonal block $A^{(k)} \in \mathbb{R}^{n_k \times n_k}$ represents the affinities within biological cluster $k$ and is given by
$A^{(k)} = p_k \cdot \mathbf{1}_{n_k \times n_k},$
with $p_k > 0$ and $ \mathbf{1}_{n_k \times n_k} \in \mathbb R^{n_k\times n_k}$ is the matrix of all ones. 
All off-diagonal blocks of $A_{0}$ are identically zero. The overall matrix $A_{0}$ is symmetric, positive semidefinite, and has rank $K$.
\end{assumption}
Note that while the off-diagonal blocks of $A_0$ are zero, this does not imply that there is no affinity between points in different clusters. 
Rather, the block-diagonal structure is a modeling choice for the purpose of selecting close data points with high similarities to the given data point.

We are still required to model the noise matrix $E$.
A standard way to model noise is using
Gaussian Orthogonal Ensembles (GOE) \cite{mehta2004random}, which are symmetric random matrices with independent Gaussian entries (up to symmetry) and are widely used to model noise. 
Importantly, their singular values concentrate around $O(\sqrt{n})$ with high probability, making them a useful model. 
However, in our setting, we are constrained to matrices with strictly positive entries, as $\tilde A$ is still a affinity matrix. 
We therefore adopt a randomized analogue with positive entries that retains the spectral scaling of GOEs. 

\begin{assumption}[Exponential Noise Matrix]\label{assump:eoe}
Let $E \in \mathbb{R}^{n \times n}$ be a symmetric random matrix with $E_{ij} = E_{ji} \sim \mathrm{Exp}(\lambda)$ independently for all $i < j$, and $E_{ii} \sim \mathrm{Exp}(\lambda/2)$, for some $\lambda > 0$. 
\end{assumption}

This ensures that $E$ is symmetric, and exhibits similar singular values as GOEs, while respecting entrywise positivity.
Our goal is to recover $A_0$ using a small subset of rows from $\widetilde{A}$.
To approximate $A_0$, we sample a subset $S \subseteq [n]$ of size $s$ and form the Nyström estimator \citep{kumar2012sampling}  
$\widehat{A} := \widetilde{A}_S^\top \widetilde{A}_{S,S}^+ \widetilde{A}_S$,  
where $\widetilde{A}_S \in \mathbb{R}^{s \times n}$ contains the sampled rows and $\widetilde{A}_{S,S} \in \mathbb{R}^{s \times s}$ is their submatrix.  
Here, $A^+$ denotes the Moore–Penrose pseudoinverse \citep{courrieu2008fast}, the unique matrix satisfying $AA^+A = A$ and $A^+AA^+ = A^+$.  
This yields a low-rank approximation of $A_0$ using the subspace defined by $S$.

How do we assess the quality of this estimator? 
The quality of reconstruction depends on how well the sampled rows capture connectivity across the data. 
Since each row of $\widetilde{A}$ represents affinities between a point and the rest of the dataset, a good sampling algorithm will sample informative rows across different regions of the data.
To formalize this, we define a notion of \emph{coverage}. Let $\text{cov}(i)$ denote the cumulative affinity received by point $i$ from previously sampled rows, given by
$
\text{cov}(i) := \sum_{j \in S_{\text{past}}} \widetilde{A}_{ji},
$
where $S_{\text{past}}$ denotes the set of indices sampled so far. Intuitively, points with low cumulative coverage have not yet been well represented in the approximation.
Sampling is performed in \emph{blocks of size $K$}.
At the start of each block, a temporary coverage vector is initialized, and the sampling distribution is uniform. 
At each step $t$ within the block, an index $i_t$ is drawn from a distribution $q^{(t)}$ that favors under-covered points, with $q^{(t)}_i \propto 1 / \text{cov}^{(t)}(i)$, where $\text{cov}^{(t)}(i)$ is the accumulated coverage over the current block. After each block, the coverage vector is reset to $0$ and sampling distribution is reset to uniform.
While Assumption~\ref{assump:eoe} ensures that the affinity matrix $\widetilde{A}$ has strictly positive entries (so $\text{cov}(i) > 0$ for all $i$), practical implementations involve sparsifying $\widetilde{A}$ (see Section~\ref{sec:balances}), which can introduce zero coverage. In such cases, we modify the sampling rule as described in Algorithm~\ref{alg:adaptive-sampling}.
This alternating scheme ensures that sampling explores new regions of the dataset and targets under-covered areas. 
We analyze this novel sampling scheme and show that the estimator $\widehat{A}$ obtained from a small number of rows approximates $A_0$ with spectral guarantees.

\subsection{Theoretical results}

We now formalize the guarantees provided by coverage-based sampling for approximating the underlying affinity matrix $A_0$.
Recall that the dataset is partitioned into $K$ biological clusters $\mathcal{C}_1, \dots, \mathcal{C}_K$ with sizes $n_1, \dots, n_K$, and total size $n = \sum_{k=1}^K n_k$. 
Let $m$ be the number of rows sampled (We will discuss how $m$ is chosen through a stopping criteria in Section~\ref{sec:balances}). 
\begin{theorem}[Cluster Coverage Guarantee]\label{thm:coverage}
Let $T_k$ denote the number of sampled rows from $\mathcal{C}_k$. 
 Then given a constant $\delta > 0$, there exists another constant $C:=C(\delta)$ such that if $m \geq Ct K \log K$, then
\(
\mathbb{P} \left( T_k \geq t, \;\forall\; k \right) \geq 1 - \delta.
\)
\end{theorem}
The proof (Appendix~\ref{app:theory}) builds on a refined coupon collector argument \citep{boneh1997coupon}. 
Unlike random sampling, the adaptive strategy achieves $\mathcal{O}(K \log K)$ coverage irrespective of cluster sizes, making it order-optimal.
We also show in Appendix~\ref{app:theory} that this property guarantees that our estimator $\widehat{A}$ closely approximates $A_0$ in operator norm. Let $\|\cdot\|_{\rm op}$ denote the operator norm.
\begin{theorem}[Spectral Approximation of $A_0$]\label{thm:spectral}
Let $\widehat{A} =  \widetilde{A}_S^\top \widetilde{A}_{S,S}^+ \widetilde{A}_S$ denote the estimator formed by low-rank reconstruction from  $m \geq C(\delta) t K \log K$ adaptively sampled rows of $\widetilde{A}$. Then,
\begin{align}
\| \widehat{A} - A_0 \|_{\rm op} \leq \frac{D n}{\sqrt{t} \cdot \min_k p_k},
\end{align}
with high probability, where $p_k$ is the cluster affinity for cluster $\mathcal{C}_k$, and $D := D(\lambda, K, \vec{p}) > 0$ depends on the parameter $\lambda$ of the noise matrix, the number of clusters $K$, and the affinities vector $\vec{p}$.
\end{theorem}
Theorem~\ref{thm:spectral} provides a guarantee on the spectral reconstruction error when enough rows are sampled. 
Specifically, when the number of samples $m = O(t K \log K)$  , i.e. satisifying the condition in Theorem~\ref{thm:coverage},  using these samples suffices to achieve small spectral error that decays at the rate $O(1/\sqrt t)$ while guaranteeing representation from all clusters.

%% file: balances.tex
Algorithm~\ref{alg:balances}  presents the full implementation of \textsc{BALANS}. 
Given a data matrix $X \in \mathbb{R}^{n \times d}$ and batch labels $\vec{b} \in \{1, \dots, B\}^n$, the algorithm also takes as input a nearest-neighbor parameter $k$, a stopping threshold $\tau$, and an adaptive sampling parameter $J$. 
Optionally, the data may be projected onto a lower-dimensional PCA subspace. 
The algorithm maintains three variables: (i) a \emph{cumulative coverage vector} $c \in \mathbb{R}^n$ tracking the cumulative affinity of each point (ii) a \emph{temporary coverage vector} $c_K \in \mathbb{R}^n$ used to compute adaptive sampling probabilities over $J$ steps and (iii) a \emph{set of sampled indices} $S$, initialized as empty. 
The main loop terminates when no new points are covered over $\tau$ consecutive iterations.

At each iteration, an index $i_t$ is sampled using the coverage-based adaptive strategy (Section~\ref{ssec:lowrank}, Alg.\ref{alg:adaptive-sampling}).
For the selected point $x_{i_t}$, a sparse affinity row $A_{i_t,:}$ is computed using the batch-dependent local scale kernel (Section\ref{ssec:bclb}, Alg.~\ref{alg:bandwidth}).
To reduce noise and memory usage, affinities are sparsified via an elbow detection heuristic \citep{truong2020selective}. Entries are first sorted in decreasing order, and only values above a sharp drop threshold are retained.
After sparsification, $\Delta$ counts how many new indices receive nonzero affinity for the first time.
If $\Delta = 0$, a stagnation counter is incremented; otherwise, it is reset.
The loop stops after $\tau$ consecutive rounds with no new coverage, serving as a convergence heuristic.
The collected affinity rows form a sparse matrix $A_S \in \mathbb{R}^{m \times n}$, which is compact and efficient to store.

In theory, the corrected data is computed via the low-rank estimator 
$\widehat{X} := (A_S^\top A_{S,S}^{+} A_S) X$, where $A_S \in \mathbb{R}^{m \times n}$ contains sampled, row-normalized affinity rows, and $A_{S,S} \in \mathbb{R}^{m \times m}$ is the submatrix over sampled indices. 
This formulation enables efficient spectral reconstruction, as $A_S$ is sparse and $A_{S,S}^{+}$ is a small dense matrix.
In practice, we omit $A_{S,S}^{+}$ to avoid its $O(m^3)$ cost, approximating the estimator as $\widehat{X} \approx (A_S^\top A_S) X$. 
Here, $A_S$ smooths $X$ , and $A_S^\top$ propagates the result to all points. 
Both matrices are row-normalized for stability, and we avoid forming the dense matrix $A_S^\top A_S$ by computing $A_S X$ followed by $A_S^\top (A_S X)$. 
Since $A_S$ is sparse, all operations remain efficient.

\begin{algorithm}[h]
\caption{BALANS Batch Correction}
\label{alg:balances}
\begin{algorithmic}[1]
\State \textbf{Input:} Data $X \in \mathbb{R}^{n \times d}$, batch labels $\vec b \in \{1,\dots,B\}^n$, parameters $k$, $\tau$, $J$
\State \textbf{Output:} Corrected data $\widehat{X} \in \mathbb{R}^{n \times d}$
\State Compute PCA on $X$ (optional); set $c, c_K \gets \mathbf{0}$; $S \gets \emptyset$; $\texttt{no\_change} \gets 0$
\While{$\texttt{no\_change} < \tau$}
    \For{$j = 1$ to $J$}
        \State Sample $i_t$ via Adaptive Sampling (Alg.~\ref{alg:adaptive-sampling}); compute $A_{i_t,:}$ via Local Scale Kernel (Alg.~\ref{alg:bandwidth})
        \State Append $A_{i_t,:}$ to $A_S$; update $S \gets S \cup \{i_t\}$
        \State $\Delta \gets \#\{j : A_{i_t,j} > 0 \text{ and } c_J[j] = 0\}$
        \If{$\Delta = 0$}
            \State $\texttt{no\_change\_count} \gets \texttt{no\_change\_count} + 1$
        \Else
            \State $\texttt{no\_change\_count} \gets 0$
        \EndIf
        \State $c \gets c + A_{i_t,:}$; \quad $c_J \gets c_J + A_{i_t,:}$
    \EndFor
    \State Reset $c_J \gets \mathbf{0}$
\EndWhile
\State Compute $\widehat{X}$ via Low-Rank Completion (Alg.~\ref{alg:completion}); \Return $\widehat{X}$
\end{algorithmic}
\end{algorithm}

Importantly, \textsc{BALANS} is robust to hyperparameters. 
Ablation studies are presented in Appendix~\ref{app:ablations}.
We now state the overall computational complexity of \textsc{BALANS} shown in Appendix~\ref{app:run_times}.

\begin{theorem}[Computational Complexity of \textsc{BALANS}]
Let $n$ be the number of data points, $d$ the feature dimension, and $m = |S|$ the number of adaptively sampled rows after BALANS converges. 
Then the total computational complexity of \textsc{BALANS} is
\(
\mathcal{O}(n m (d + \log n)),
\)
where the first term accounts for sparse affinity construction and matrix multiplications, and the second term covers nearest neighbor search. 
When $m \ll n$, this is approximately linear in $n$.
\end{theorem}
By Theorem~\ref{thm:spectral}, sampling $m = \mathcal{O}(K \log K)$ rows suffices to reconstruct $A_0$ with high probability while covering a large fraction of points, allowing the algorithm to terminate. 
This results in an overall computational complexity of $\mathcal{O}(n K \log K (d + \log n))$, which is near-linear in the dataset size $n$.

%% file: experiments.tex
\begin{table*}[!t]
\centering
\caption{
Evaluation scores across three scenarios: \textbf{JUMP 2} (large, diverse batches), \textbf{DEEP 1} (low batch diversity, DeepProfiler), and \textbf{JUMP 1}. Since most methods are deterministic, the Standard Deviation (SD) are $0$;  we omit it here due to space. Tables with SDs are in the Appendix.
}

{\scriptsize
\setlength{\tabcolsep}{6pt}
\begin{tabular}{lcccccccccc}
\toprule
Method    & Conn. & LISI-batch & Silh-batch & LISI-label & ARI  & NMI  & Silh-label & Avg-batch & Avg-label & Avg-all \\
\midrule
BALANS    & 0.33 & \textbf{0.48} & \textbf{0.91} & \textbf{1.00} & 0.01 & \textbf{0.46} & 0.32 & \textbf{0.57} & \textbf{0.45} & \textbf{0.50} \\
Scanorama & \textbf{0.34} & 0.41 & 0.83 & \textbf{1.00} & 0.01 & 0.30 & 0.27 & 0.53 & 0.40 & 0.45 \\
SCVI      & 0.28 & 0.44 & 0.82 & 0.99 & 0.01 & 0.28 & 0.29 & 0.53 & 0.39 & 0.44 \\
fastMNN   & \textbf{0.34} & 0.45 & 0.79 & \textbf{1.00} & \textbf{0.02} & 0.26 & 0.22 & 0.53 & 0.38 & 0.44 \\
Harmony   & \textbf{0.34} & 0.25 & 0.84 & \textbf{1.00} & 0.01 & 0.27 & 0.32 & 0.48 & 0.40 & 0.43 \\
Sphering  & \textbf{0.34} & 0.00 & 0.84 & \textbf{1.00} & 0.00 & 0.22 & \textbf{0.38} & 0.39 & 0.40 & 0.40 \\
Combat    & \textbf{0.34} & 0.02 & 0.81 & \textbf{1.00} & 0.00 & 0.24 & 0.31 & 0.39 & 0.39 & 0.39 \\
Baseline  & \textbf{0.34} & 0.01 & 0.80 & \textbf{1.00} & 0.00 & 0.23 & 0.31 & 0.38 & 0.39 & 0.38 \\
DESC      & 0.29 & 0.40 & 0.80 & 0.99 & 0.00 & 0.25 & 0.28 & 0.50 & 0.39 & 0.42 \\
\bottomrule
\end{tabular}

\label{tab:table2}
}
{\scriptsize
\setlength{\tabcolsep}{5.7pt}
\begin{tabular}{lcccccccccc}
\toprule
Method    & Conn. & LISI-batch & Silh-batch & LISI-label & ARI  & NMI  & Silh-label & Avg-batch & Avg-label & Avg-all \\
\midrule
BALANS     & 0.20 & 0.61 & 0.89 & \textbf{1.00} & \textbf{0.00} & 0.32 & 0.39 & \textbf{0.57} & 0.43 & \textbf{0.49} \\
Sphering   & \textbf{0.30} & 0.43 & \textbf{0.90} & 0.99 & \textbf{0.00} & 0.30 & \textbf{0.46} & 0.54 & \textbf{0.44} & 0.48 \\
Seurat RPCA & 0.26 & \textbf{0.62} & 0.79 & 0.99 & \textbf{0.00} & 0.25 & 0.36 & 0.56 & 0.40 & 0.47 \\
Harmony     & 0.26 & \textbf{0.62} & 0.80 & 0.99 & \textbf{0.00} & 0.23 & 0.37 & 0.56 & 0.40 & 0.47 \\
Seurat CCA  & 0.27 & 0.60 & 0.79 & 0.99 & \textbf{0.00} & 0.27 & 0.36 & 0.55 & 0.40 & 0.47 \\
Baseline    & 0.27 & 0.57 & 0.80 & 0.99 & \textbf{0.00} & 0.23 & 0.37 & 0.55 & 0.40 & 0.46 \\
Combat      & 0.27 & 0.58 & 0.80 & 0.99 & \textbf{0.00} & 0.24 & 0.37 & 0.55 & 0.40 & 0.46 \\
fastMNN     & 0.26 & 0.60 & 0.78 & 0.99 & \textbf{0.00} & 0.24 & 0.35 & 0.55 & 0.40 & 0.46 \\
Scanorama   & 0.19 & 0.56 & 0.80 & \textbf{1.00} & \textbf{0.00} & \textbf{0.34} & 0.33 & 0.52 & 0.42 & 0.46 \\
SCVI        & 0.24 & 0.58 & 0.82 & 0.99 & \textbf{0.00} & 0.29 & 0.34 & 0.55 & 0.41 & 0.47 \\
DESC        & 0.23 & 0.55 & 0.79 & 0.99 & \textbf{0.00} & 0.26 & 0.33 & 0.53 & 0.40 & 0.46 \\
\bottomrule
\end{tabular}
}
{\scriptsize
\setlength{\tabcolsep}{5pt}
\begin{tabular}{lccccccccccc}
\toprule
Method      & Conn. & KBET & LISI-batch & Silh-batch & LISI$_l$ & ARI  & NMI  & Silh-label & Avg-batch & Avg-label & Avg-all \\
\midrule
Seurat CCA  & 0.59 & \textbf{0.61} & 0.50 & 0.88 & \textbf{0.98} & \textbf{0.05} & 0.40 & 0.47 & \textbf{0.65} & 0.48 & \textbf{0.56} \\
BALANS      & 0.54 & 0.45 & 0.44 & \textbf{0.89} & \textbf{0.98} & \textbf{0.05} & \textbf{0.46} & \textbf{0.53} & 0.58 & \textbf{0.51} & 0.54 \\
Seurat RPCA & 0.59 & 0.46 & 0.41 & 0.88 & \textbf{0.98} & 0.03 & 0.39 & 0.47 & 0.59 & 0.47 & 0.53 \\
fastMNN     & 0.53 & \textbf{0.61} & 0.46 & 0.83 & 0.97 & 0.03 & 0.35 & 0.43 & 0.61 & 0.45 & 0.53 \\
Harmony     & \textbf{0.60} & 0.41 & 0.43 & 0.88 & \textbf{0.98} & 0.03 & 0.39 & 0.47 & 0.58 & 0.47 & 0.52 \\
Scanorama   & 0.35 & 0.56 & \textbf{0.52} & 0.80 & \textbf{0.98} & 0.03 & 0.34 & 0.43 & 0.56 & 0.45 & 0.50 \\
SCVI        & 0.50 & 0.40 & 0.45 & 0.85 & 0.97 & 0.03 & 0.38 & 0.46 & 0.58 & 0.47 & 0.51 \\
Baseline    & 0.51 & 0.17 & 0.13 & 0.78 & \textbf{0.98} & 0.02 & 0.32 & 0.47 & 0.40 & 0.45 & 0.42 \\
Combat      & 0.54 & 0.08 & 0.11 & 0.80 & \textbf{0.98} & 0.02 & 0.33 & 0.47 & 0.38 & 0.45 & 0.42 \\
Sphering    & 0.48 & 0.10 & 0.06 & 0.79 & \textbf{0.98} & 0.01 & 0.28 & 0.47 & 0.36 & 0.44 & 0.40 \\
DESC        & 0.49 & 0.35 & 0.42 & 0.84 & 0.97 & 0.02 & 0.36 & 0.45 & 0.57 & 0.46 & 0.50 \\
\bottomrule
\end{tabular}
}
\end{table*}
\begin{figure*}[!t]
    \centering
    \includegraphics[width=\textwidth]{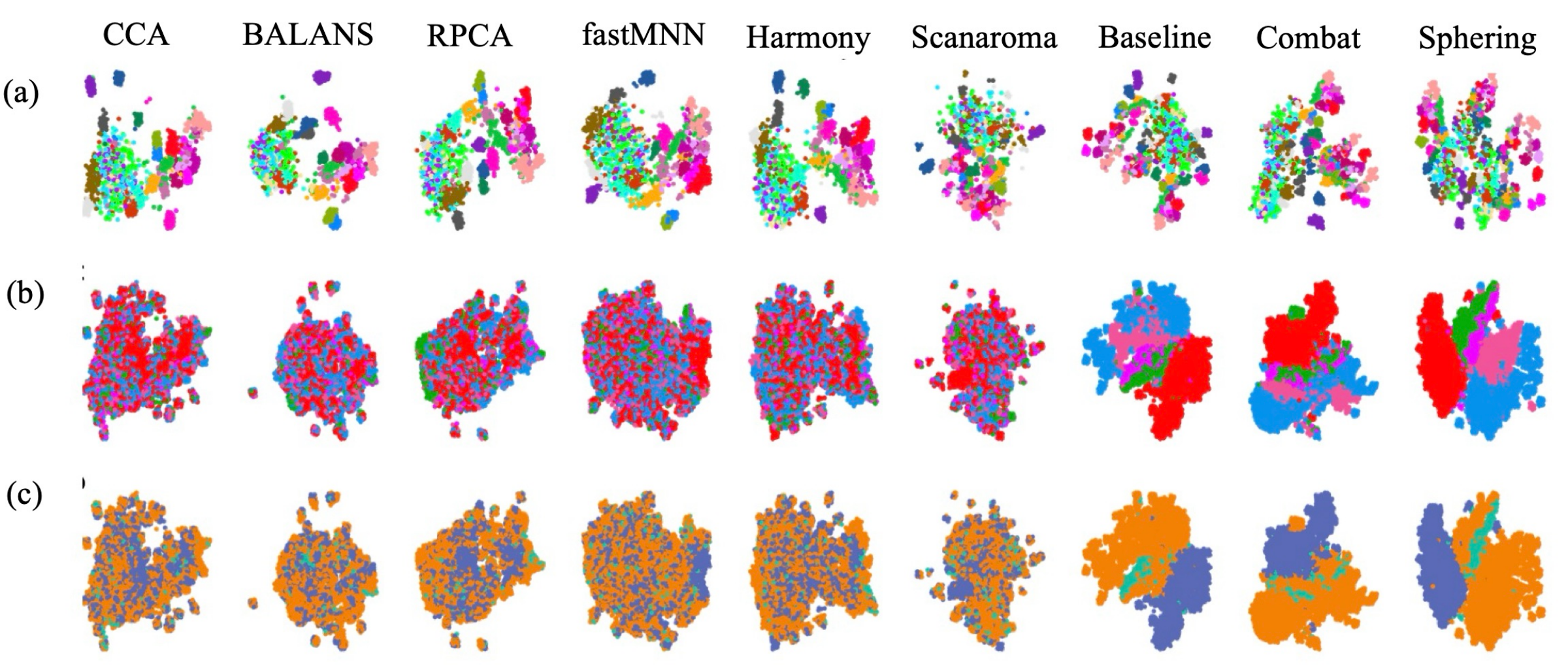}
    \caption{
       Figure 3: UMAP visualizations for JUMP~1. (a) Colored by biological cluster, illustrating separation across a few representative compounds. (b) Colored by source lab. (c) Colored by microscope. These panels demonstrate how batch effects manifest across different technical sources. Notably, BALANS yields a tight and coherent biological clusters while substantially reducing separation across labs.
    }
    \label{fig:umap_panel}
\end{figure*}
\begin{table*}[!tbp]
\centering
\caption{Wall-clock runtimes (in minutes) for best performing methods across real-world scenarios. 
}
\setlength{\tabcolsep}{4pt}
{\small
\begin{tabular}{lccc}
\toprule
\textbf{Method} & \textbf{Scenario JUMP 1} & \textbf{Scenario JUMP 2} & \textbf{Scenario DEEP 1} \\
\midrule
BALANS         &  \textbf{00:00:59}  & \textbf{00:55:30} & \textbf{00:00:44} \\
Harmony         & 00:12:35   & 01:44:53  & 00:22:43 \\
Scanorama       & 00:01:33  & 01:13:02          & 00:01:57  \\
Seurat CCA **     & 00:34:37 & 12:00:00+         & 01:03:24 \\
Seurat RPCA **    & 00:14:03 & 12:00:00+          & 00:21:29 \\
\bottomrule
\label{tab:wallclock}
\end{tabular}
}

\end{table*}

\begin{figure*}[!tbp]
    \centering
    \includegraphics[width=0.8\textwidth]{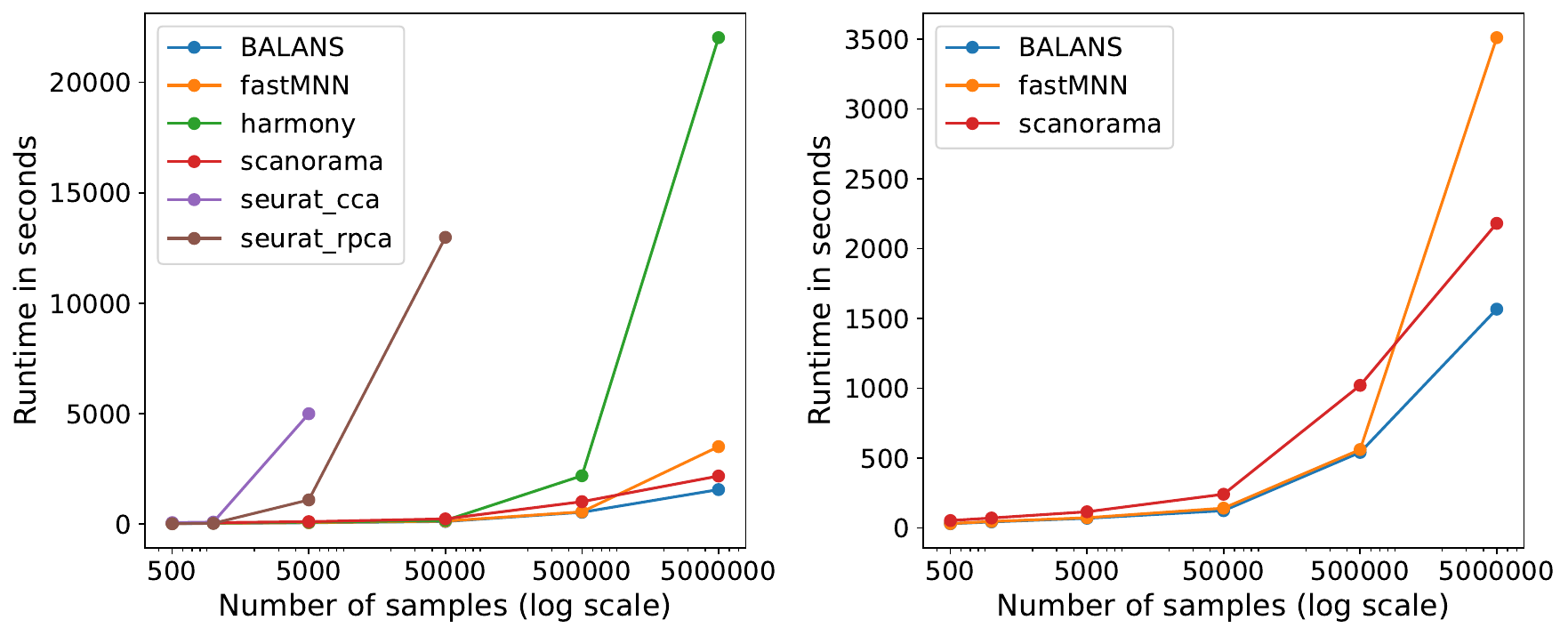}
    \caption{
Runtime comparison on synthetic data (10 compounds, 5 batches, 10 features). 
Left: all methods. Right: fastest methods only. 
X-axis: number of data points (log scale); Y-axis: runtime in seconds. 
    }
    \label{fig:runtime_comparison}
\end{figure*}

We evaluate \textsc{BALANS} on 8 real-world scenarios from 7 datasets, spanning both CellProfiler and DeepProfiler features. 
Five scenarios are from the JUMP Cell Painting Consortium (approximately 140{,}000 perturbations, 115~TB of images), and three use DeepProfiler embeddings from BBBC datasets (approximately 1000 perturbations with 50000 well averaged profiles). 
We also include 6 synthetic scenarios based on Gaussian mixtures ranging from 500 to 5 million data points. 
All methods are evaluated using metrics that jointly assess batch correction and biological signal preservation. 
Full details are provided in Appendix~\ref{app:dataset}.
We compare \textsc{BALANS} against a broad range of batch correction methods, including statistical approaches (Combat \citep{zhang2020combat}, Sphering \citep{kessy2018optimal}), linear projection methods (CCA, RPCA \citep{satija2015spatial}), and graph-based or KNN-based methods (FastMNN , Scanorama \citep{hie2019efficient}, Harmony \citep{korsunsky2019fast}). 
These methods differ in their underlying assumptions.
Table~\ref{tab:method_comparison} in Appendix~\ref{app:methods} summarizes these properties; full implementation details and references are provided in Appendix~\ref{app:methods}.

We evaluate performance using a suite of established metrics \citep{luecken2022benchmarking}. 
These include Graph Connectivity, kBET, LISI (batch and label), Silhouette scores (batch and label), and clustering-based measures such as Leiden ARI and Leiden NMI. 
Each metric is normalized to the range [0, 1], with higher values indicating better performance. 
We note that kBET could not be computed for large-scale settings (e.g., JUMP 2) due to resource constraints. 
Full metric definitions and implementation details are provided in Appendix~\ref{app:metrics}.
Table~\ref{tab:metrics_summary} in Appendix~\ref{app:metrics} provides a summary of the evaluation metrics.
We highlight three scenarios here in Table~\ref{tab:table2} (full details in Appendix~\ref{app:realworld}).
\textbf{JUMP 1}: medium-scale CellProfiler (25K points, 3 sources, 302 compounds). 
\textbf{JUMP 2}: large-scale, high-diversity (350K points, 5 microscopes, 80K+ compounds, low overlap). 
\textbf{DEEP 1}: DeepProfiler version of \texttt{BBBC036} with batch effects based on groups of plate rows. 
We also compare wall-clock runtimes (Table~\ref{tab:wallclock}) for top-performing methods. 
\textsc{BALANS} outperforms other competing methods across datasets.

\textbf{Explaination of metrics:} We distinguish between batch metrics and bio metrics to provide a balanced view of integration quality. 
Batch metrics (e.g., Graph Connectivity, LISI-batch) assess how well spurious technical variation is removed, but can be misleading in isolation. Methods that erase all structure (e.g., by adding noise) may score highly despite discarding biological signal.
Label metrics (e.g., ARI, NMI, Silhouette-label), in contrast, directly evaluate whether meaningful biological structure is preserved, and thus carry greater interpretive weight. 
Because no single metric captures all aspects, we emphasize comparative patterns and rely on Avg-all as a more robust summary across tradeoffs between local versus global structure. 
Ultimately, we prioritize improvements in label metrics, since recovering biologically coherent variation is the primary goal of batch correction.
Notably, BALANS is consistently performs the best across all datasets on the average of label-based bio metrics (Avg-label), underscoring its ability to preserve biologically coherent variation.
A detailed discussion on metrics is provided in Appendix~\ref{app:dataset}.

\textbf{Interpretation of metrics:} As an illustration of biological interpretability, BALANS more clearly recovers positive control compounds with known phenotypes. 
For example, the Aurora kinase inhibitor AMG900 \cite{} forms a distinct cluster with a high silhouette score (0.82 vs. 0.45 for Harmony). 
This highlights BALANS’s ability to preserve meaningful biological structure. Additional compound-level comparisons and heatmaps are provided in Appendix~\ref{app:poscomp}.

 BALANS performs the best on highly heterogeneous datasets, compared to baselines. 
In particular, the results on JUMP-1 and JUMP-2 illustrate this clearly. JUMP-2, which contains stronger batch effects and higher biological and technical heterogeneity, shows a clear advantage for BALANS across all metrics. On JUMP-1, where heterogeneity is more moderate, BALANS provides stable correction without substantially surpassing all baselines. 
Together, these results indicate that BALANS performs especially well as dataset heterogeneity increases, underscoring its suitability for large and diverse real-world settings.

To evaluate scalability, we introduce a synthetic benchmark based on a structured Gaussian mixture model. 
It supports datasets with up to 5 million points and varying batch complexity (details in Appendix~\ref{app:synthetic}). 
The main experiment uses 10 compound clusters across 5 batches in 10 dimensions. 
As shown in Figure~\ref{fig:runtime_comparison}, BALANS achieves the fastest runtimes even on 5 million data points without compromising correction quality.
Aggregate Scores for this experiment,
additional scenarios, runtime details (including GPU hardware), and their aggregate scores are in Appendix~\ref{app:synthetic}.

%% file: conclusions.tex
We introduce BALANS, a novel method that integrates local affinity structure with adaptive sampling to provide a scalable and theoretically grounded approach to batch correction in Cell Painting data. 

%% file: theory.tex
\subsection{Proof of Theorem~\ref{thm:coverage}}

\makeatletter
\setcounter{theorem}{0} 
\renewcommand{\thetheorem}{1}  
\begin{theorem}[Cluster Coverage Guarantee, restated]
Let $T_k$ denote the number of sampled rows from cluster $\mathcal{C}_k$. 
Then for any constant $\delta > 0$, there exists a constant $C := C(\delta)$ such that if $m \geq Ct K \log K$, then (under adaptive sampling~\ref{alg:adaptive-sampling})
\[
\mathbb{P} \left( T_k \geq t \;\text{for all}\; k \in [K] \right) \geq 1 - \delta.
\]
\end{theorem}
\addtocounter{theorem}{-1}  
\renewcommand{\thetheorem}{\arabic{section}.\arabic{theorem}} 
\makeatother

Before proving this theorem we state a few mild assumptions.

\begin{assumption}[Growing–cluster regime]
\label{assump:growing-clusters}
Fix an integer $t \ge 1$ that is independent of $n$.  
Let $\{n_k\}_{k=1}^K$ be the cluster sizes and assume 
\[
\min_{1 \le k \le K} n_k \to \infty 
\quad \text{as the total sample size } n = \sum_{k=1}^K n_k \to \infty.
\]
\end{assumption}

\noindent
This is a very mild assumption: it allows highly imbalanced cluster sizes, including cases where one or more clusters are vanishingly small relative to $n$, as long as each cluster still grows without bound.

\vspace{1em}

\begin{assumption}[Low-noise regime]
\label{assump:low-noise}
Assume the noise level satisfies
\[
p_{\min} \cdot \lambda > \delta n
\]
for some constant $\delta > 0$ independent of $n$. 
\end{assumption}

\noindent
Intuitively, it ensures that true affinities are not overwhelmed by noise.

\begin{proof}

Recall that 
\[
\tilde A = A_{0} + E,
\]

is the noisy block affinity matrix that we sample rows from.
Let $m$ denote the total number of sampling steps, performed in blocks of size $K$. 
Consider one such block, and let $s \in [K]$ denote the $s^{\rm th}$ sampling step within the block. 
Define the coverage of point $x_i$ at step $s$ of this block as
\[
\text{cov}^{(s)}(i) := \sum_{j \in S_{\text{past}}} \widetilde{A}_{ji},
\]
where $S_{\text{past}} := \{ i_j \mid 1 \leq j < s \}$ denotes the set of indices sampled in the current block  prior to step $s$.
Let 
\[
\operatorname{cov}^{(s)} := \bigl( \operatorname{cov}^{(s)}(1), \operatorname{cov}^{(s)}(2), \dots, \operatorname{cov}^{(s)}(n) \bigr)
\]
be the coverage vector at step $s$.
The inverse-coverage sampling distribution is given by 
\[q^s := (q^{(s)}(1), q^{(s)}(2), \dots q^{(s)}(n))\]
at step $s$ is defined as
\[
q^{(s)}(i) := \frac{ \operatorname{cov}^{(s)}(i)^{-1} }{ \sum_{j=1}^n \operatorname{cov}^{(s)}(j)^{-1} },
\]
whenever all entries all positive $\operatorname{cov}^{(s)}(j) > 0$.
While this expression is well-defined in practice (since there is always a positive affinity between any two points in practice), theoretically there may exist cases e.g., $E = 0$, where one or more entries satisfy $\text{cov}^{(s)}(j) = 0$. 
In such cases, $\text{cov}^{(s)}(j)^{-1} = \infty$, making the expression undefined in the standard sense.
Suppose $\operatorname{cov}^{(s)}(i) = 0$ for a subset of indices and $\operatorname{cov}^{(s)}(i) > 0$ for the rest.
We define the sampling distribution $q^s$ to be
\[
q^{(s)}_i := 
\begin{cases}
\frac{1}{|\{ j : \text{cov}^{(s)}(j) = 0 \}|} & \text{if } \text{cov}^{(s)}(i) = 0, \\
0 & \text{otherwise}.
\end{cases}
\]

That is, we sample uniformly among the uncovered points. 
This corresponds to the limiting behavior of the inverse-coverage distribution as uncovered entries dominate. 
Now to build intuition we consider the noise-free case, when $E = 0$.

\paragraph{Noise-free case.}

Assume $E=0$. Therefore the observed affinity matrix $\widetilde A$ is given by
\begin{align}
\widetilde A = A_0.
\end{align}
Recall that the matrix $A_0\in\mathbb R^{n\times n}$ is block–diagonal with $K$ disjoint clusters  
$\mathcal C_1,\dots,\mathcal C_K$, i.e.
\[
A_{0} := \mathrm{diag}(A^{(1)}, A^{(2)}, \dots, A^{(K)}),
\]
where each diagonal block $A^{(k)} \in \mathbb{R}^{n_k \times n_k}$ represents the affinities within cluster $k$ and is given by
$A^{(k)} = p_k \cdot \mathbf{1}_{n_k \times n_k},$
with $p_k > 0$ and $ \mathbf{1}_{n_k \times n_k} \in \mathbb R^{n_k\times n_k}$ is the matrix of all ones. 
Therefore when $E = 0$, for $i\in[n]$ and $j\in[n]$
\begin{align}
\widetilde A_{i,j}>0 \quad\Longleftrightarrow\quad x_i,x_j\text{ belong to the same cluster}.
\end{align}

Consider a sampling block and the current step in that block $s\in[K]$.  Write the sequence of indices drawn in this block as  
$
i_1,\dots,i_K.
$

In the noiseless case, due to the block-diagonal matrix structure of $A_0$, only points within the same cluster as the previously sampled indices receive nonzero coverage. 
Therefore, at every step $s > 1$, the uncovered set $Z^{(s)}$ consists of points in clusters from which no indices have yet been sampled.

Hence, the sampling distribution is simply uniform over $\{ j : \text{cov}^{(s)}(j) = 0 \}$ given by
\[
q^{(s)}_i := 
\begin{cases}
\frac{1}{|\{ j : \text{cov}^{(s)}(j) = 0 \}|} & \text{if } \text{cov}^{(s)}(i) = 0, \\
0 & \text{otherwise}.
\end{cases}
\]

Let $\mathcal{U}^{(s)}$ denote the set of clusters that have not yet been sampled by step $s$, i.e.,
\[
\mathcal{U}^{(s)} := \left\{\, k \in [K] : \mathcal{C}_k \cap \{i_1, \dots, i_{s-1}\} = \varnothing \,\right\}.
\]
By construction, $\mathcal{U}^{(1)} = [K]$ and exactly one new cluster is removed from $\mathcal{U}^{(s)}$ at each step. Therefore,
\[
|\mathcal{U}^{(s)}| = K - (s - 1),
\]
and the sampling distribution at step $s$ is uniform over the union of these remaining $|\mathcal{U}^{(s)}|$ clusters. 
After $K$ steps, $\mathcal{U}^{(K+1)} = \varnothing$, and each cluster has contributed exactly one sample.

Each block samples one row from each of the $K$ clusters. 
Note that there are $m/K$ blocks. 
Therefore the question of sampling atleast $t$ unique samples from each cluster is equivalent to the question of how many samples we require to sample atleast $t$ unique blocks from each cluster when we over $m/K$ draws uniformly at random with replacement.
This is reminiscent of the classic coupon collector problem \cite{boneh1997coupon}.
More formally we state assumptions and prove the following results.

\begin{lemma}[Coupon collector with a partial target, \citep{boneh1997coupon}]
\label{lem:partial-coupon}
Let $N \ge t$ and sample i.i.d.\ coupons from the set $[N] = \{1,\dots,N\}$ with replacement.  
Let $S_t$ be the number of draws required to observe $t$ distinct coupons. 
Define the $N^{\rm th}$ Harmonic number as
\[
H_N = \sum_{i=1}^N 1/i
\]

Then,
\begin{align*}
\mathbb{E}\left[S_t\right] 
&= N \left( H_N - H_{N - t} \right)
= N \sum_{i = N - t + 1}^{N} \frac{1}{i}
= t + \mathcal{O}\left( \frac{t^2}{N} \right).
\end{align*}
Moreover, for every $\gamma > 0$,
\begin{align}
\Pr\left( S_t > (1 + \gamma) N \left( H_N - H_{N - t} \right) \right) 
\le \exp\left( -\frac{ \gamma^2 t }{ 2(1 + \gamma) } \right).
\end{align}
\end{lemma}

\begin{proof}
We sketch the proof here. 
For full proof refer to \citet{boneh1997coupon}.

We can write $S_t = \sum_{j = 0}^{t - 1} X_j$, where $X_j$ is the waiting time to observe a new coupon after $j$ distinct ones have been collected.  
Each $X_j$ is geometric with success probability $p_j = 1 - \frac{j}{N}$, so
\[
\mathbb{E}[X_j] = \frac{1}{p_j}, \quad 
\operatorname{Var}(X_j) = \frac{1 - p_j}{p_j^2} \le \frac{1}{p_j^2}.
\]
Then by linearity,
\[
\mathbb{E}[S_t] = \sum_{j = 0}^{t - 1} \frac{1}{1 - \frac{j}{N}} 
= N \sum_{i = N - t + 1}^{N} \frac{1}{i}.
\]
A Bernstein or Chernoff bound for geometric random variables yields the tail inequality.
\end{proof}

Within each block we draw one sample from every cluster,  
so each cluster $k$ is sampled $m_k = m / K$ times after $m$ total draws.
Fix $\gamma > 0$ and set
\begin{align}
m_k &= \left\lceil (1 + \gamma) n_k \left( H_{n_k} - H_{n_k - t} \right) \right\rceil, \\
m &= K \cdot m_k.
\end{align}
Then, under Assumption~\ref{assump:growing-clusters},
\begin{align}
\Pr\left( \text{any cluster has fewer than } t \text{ distinct samples} \right)
\le K \cdot \exp\left( -\frac{ \gamma^2 t }{ 2(1 + \gamma) } \right).
\end{align}
The failure probability becomes $O\left( \exp\left( -\gamma t \right) \right),$ for a constant $\gamma$.
Since $t$ is fixed and $n_k \to \infty$, we have
\begin{align}
n_k \left( H_{n_k} - H_{n_k - t} \right)
= t + \mathcal{O}\left( \frac{t^2}{n_k} \right)
= t \left( 1 + o(1) \right).
\end{align}
Therefore,
\begin{align}
m &= K (1 + \gamma) t (1 + o(1)) = O(K t),
\end{align}
samples suffice i.e., the total number of blocks growing linearly with $K$ (and $t$), independent of the cluster sizes, suffices to guarantee atleast $t$ samples from each cluster with probability $1 - \exp(-\gamma t)$.

\begin{remark}
This concentration is also tight as an upper bound. 
In particular, we can expect with high probability that the number of samples drawn from each cluster does not exceed $O(t)$.
This upper bound will be used in the proof of the spectral approximation guarantee (Theorem~\ref{thm:spectral}).
\end{remark}

\paragraph{Noisy case.}

We now consider the noisy case where $E \ne 0$.  
This setting is more challenging, as the sampling procedure may no longer guarantee exactly one sample from each cluster in every block.  
To address this issue, we show that despite the noise, there remains a constant probability (independent of $n$) that each block includes at least one sample from every cluster.  
This key property enables our analysis to proceed analogously to the noiseless case.

We now analyze the noisy case under Assumption~\ref{assump:eoe}, where the entries of the noise matrix satisfy $E_{ij} = E_{ji} \sim \operatorname{Exp}(\lambda)$ i.i.d.\ and are independent of the sampling procedure.
Consider $v$ consecutive sampling blocks, each of size $K$.  
Let $T_k^v$ denote the number of \emph{distinct} samples drawn from cluster $\mathcal{C}_k$ across these $v$ blocks.

Define the event $\mathcal{E}_k := \{T_k^v = 0\}$, i.e., cluster $\mathcal{C}_k$ is never selected in any of the $v$ blocks.  
Let $\mathcal{E} := \bigcup_{k=1}^K \mathcal{E}_k$ denote the failure event that \emph{some} cluster is missed entirely.

By the union bound,
\begin{align}
\Pr(\mathcal{E}) \le K \cdot \max_{1 \le k \le K} \Pr(\mathcal{E}_k).
\end{align}

To bound $\Pr(\mathcal{E}_k)$, note that
\begin{align}
T_k^v \ge \sum_{j=1}^{v} Z_j,
\end{align}
where $Z_j := \mathbf{1}\left\{ \text{cluster } \mathcal{C}_k \text{ is selected in block } j \right\}$ is an indicator variable.
This implies that  
\[
\Pr(\mathcal E_k) = \Pr\left(\sum_{j=1}^{v} Z_j = 0\right) = \Pr(Z_1 = 0)^v,
\]

since $Z_j$ is iid.
Let us now compute the $\Pr(Z_1 = 0)$.
Here 
\begin{align}
&\Pr(Z_1 = 0) = \Pr(\mathcal C_k \text{ is not selected in block }1) \nonumber \\
&= \Pr\left(\bigcap_{s=1}^K\middle\{\mathcal C_k \text{ is not selected in step }s\middle| \mathcal C_k \text{ is not selected in steps }1\text{ to }s-1\middle\}\right) \nonumber \\
&\leq \Pr\left(\middle\{\mathcal C_k \text{ is not selected in step }K\middle| \mathcal C_k \text{ is not selected in steps }1\text{ to }K-1\middle\}\right).
\end{align}
Define 
\[
\Pr\left(\middle\{\mathcal C_k \text{ is not selected in step }K\middle| \mathcal C_k \text{ is not selected in steps }1\text{ to }K-1\middle\}\right) : = r_k
\]

Fix cluster $\mathcal{C}_k$ and condition on the event that it was not selected in the first $K-1$ steps of a block. 
Let $\mathcal{S}_{\text{past}} = \{i_1, \dots, i_{K-1}\}$ be the sampled indices so far, with $\mathcal{S}_{\text{past}} \cap \mathcal{C}_k = \emptyset$. 
Recall the coverage at step $K$ as
\[
\operatorname{cov}^{(K)}(i) := \sum_{j \in \mathcal{S}_{\text{past}}} \widetilde{A}_{ji} = \sum_{s=1}^{K-1} (A_{0,\,i_s i} + E_{i_s i}).
\]
Since $A_0$ is block-diagonal and $\mathcal{C}_k$ was not selected in the past, we have:
\[
\operatorname{cov}^{(K)}(i) = 
\begin{cases}
\sum_{s=1}^{K-1} E_{i_s i} := S_i, & i \in \mathcal{C}_k, \\
\sum_{s=1}^{K-1} (p_{c(i)} + E_{i_s i}) = S_i + p_{c(i)}, & i \notin \mathcal{C}_k,
\end{cases}
\]
where $p_{c(i)} \in [p_{\min}, p_{\max}]$ and $S_i \sim \operatorname{Gamma}(K-1, \lambda)$ are independent across $i$.

\begin{remark}
In practice, the noise contributions from rows outside the current block should also be accounted for when computing the p,f. 
However, these terms can be incorporated straightforwardly and do not affect the core arguments. 
For clarity and brevity, we omit them from the main analysis.
\end{remark}

The inverse-coverage sampling probability at step $K$ is
\[
q^{(K)}(i) = \frac{\operatorname{cov}^{(K)}(i)^{-1}}{\sum_{j=1}^n \operatorname{cov}^{(K)}(j)^{-1}}.
\]
Therefore, the probability that $\mathcal{C}_k$ is not selected at step $K$ (i.e., the conditional miss probability) is given by
\begin{align}
r_k 
\leq \mathbb{E}\left[\frac{\sum_{i \notin \mathcal{C}_k} (S_i + p_{c(i)})^{-1}}{\sum_{i \in \mathcal{C}_k} S_i^{-1} + \sum_{i \notin \mathcal{C}_k} (S_i + p_{c(i)})^{-1}}\right].
\end{align}

Using $p_{c(i)} \in [p_{\min}, p_{\max}]$, we obtain the bound
\begin{align}
r_k \le 
\mathbb{E}\left[\frac{\sum_{i \notin \mathcal{C}_k} (S_i + p_{\min})^{-1}}{\sum_{i \in \mathcal{C}_k} S_i^{-1} + \sum_{i \notin \mathcal{C}_k} (S_i + p_{\min})^{-1}}\right].
\end{align}

For all $i \notin \mathcal{C}_k$, we have $(S_i + p_{\min})^{-1} \le p_{\min}^{-1}$. Therefore,
\[
\sum_{i \notin \mathcal{C}_k} (S_i + p_{\min})^{-1} \le \frac{n - n_k}{p_{\min}} \le \frac{n}{p_{\min}}.
\]

We now lower bound $\sum_{i \in \mathcal{C}_k} S_i^{-1}$. Since $S_i \sim \operatorname{Gamma}(K-1, \lambda)$, we have $\mathbb{E}[S_i^{-1}] = \lambda / (K-2)$ for $K \ge 3$. Thus,
\[
\mathbb{E}\left[\sum_{i \in \mathcal{C}_k} S_i^{-1} \right] = \frac{n_k \lambda}{K-2}.
\]

\begin{lemma}[Concentration of inverse-gamma sums]
\label{lem:inv-gamma-concentration}
Let $S_1, \dots, S_{n_k}$ be i.i.d.\ random variables with $S_i \sim \mathrm{Gamma}(\alpha, \lambda)$ for some $\alpha > 2$ and $\lambda > 0$. Then, with probability at least $1 - \exp(-c n_k)$ for some constant $c > 0$ depending only on $\alpha$, we have
\[
\sum_{i=1}^{n_k} \frac{1}{S_i} \ge \frac{1}{2} \cdot \frac{n_k \lambda}{\alpha - 1}.
\]
\end{lemma}

\noindent
This uses the facts that $\mathbb{E}[S_i^{-1}] = \frac{\lambda}{\alpha - 1}$ and $\operatorname{Var}(S_i^{-1}) = \frac{\lambda^2}{(\alpha - 1)^2 (\alpha - 2)}$, and applies standard concentration for sums of i.i.d.\ sub-exponential variables.
From the concentration bound above,
\[
\sum_{i \in \mathcal{C}_k} S_i^{-1} \ge \frac{1}{2} \cdot \frac{n_k \lambda}{K-2}.
\]
  
On the high-probability event that the denominator is at least $\frac{n_k \lambda}{2(K-2)}$, we get:
\[
r_k \le \frac{\frac{n}{p_{\min}}}{\frac{n_k \lambda}{2(K-2)}} + \frac{\delta}{2} = \frac{2n (K-2)}{p_{\min} \lambda n_k} + \frac{\delta}{2}.
\]

To ensure $r_k \le \delta$, it suffices to require
\[
\frac{2n (K-2)}{p_{\min} \lambda n_k} \le \frac{\delta}{2} \quad \Rightarrow \quad p_{\min} \lambda \ge \frac{4n (K-2)}{\delta n_k},
\]
which is true due to Assumption~\ref{assump:low-noise}

Now fix $v = C\log K$. Then,
\[
\Pr(\mathcal{E}_k) \le \delta^{\log K} = K^{\log \delta} = K^{-\gamma}, \quad \text{where } \gamma := -C\log \delta > 0.
\]

Choosing $C$ such that $-C\log{\delta} > 1$.
By the union bound,
\[
\Pr(\mathcal{E}) \le K \cdot \max_k \Pr(\mathcal{E}_k) \le K \cdot K^{-\gamma} = K^{1 - \gamma}.
\]
Since $\gamma > 1$, we have $\Pr(\mathcal{E}) \to O(K^{1- \gamma})$.

In the presence of noise, each block of size $K$ no longer guarantees exactly one sample from each cluster. 
However we showed that with high probability that $C \log K$ such blocks suffice to cover all $K$ clusters.

Once a cluster is selected in a block, the sample from that cluster is uniformly drawn from its $n_k$ members. 
Thus, conditioned on a cluster being selected, the sampling of rows from that cluster proceeds as uniform-with-replacement from $[n_k]$.

We now repeat the coupon collector argument within each cluster. 
Let $m_k = C \log K$ denote the number of blocks in which cluster $\mathcal{C}_k$ is selected (so $m_k$ is random, but with high probability $m_k \ge C \log K$ for all $k$). 

To obtain $t$ distinct rows from each cluster, it suffices to ensure that at least $t$ distinct elements are drawn uniformly from $[n_k]$ over $m_k$ such trials. This is the same setting as Lemma~\ref{lem:partial-coupon}.
By that result, it suffices to have
\[
m_k \ge (1 + \gamma) n_k (H_{n_k} - H_{n_k - t}) \approx (1 + \gamma) t
\]
with failure probability at most $\exp(-\gamma t / 2)$. Since $m_k = C \log K$ with high probability and $t$ is fixed, this condition holds for sufficiently large $K$.

Putting this together, we conclude that
\[
m = K \cdot C (1+\gamma) t \log K 
\]
total samples suffice to obtain $t$ distinct rows from each cluster with high probability, even in the presence of noise. 
This mirrors the noiseless case, except that we now require $\log K$ blocks (instead of $K$ samples directly) to ensure cluster coverage.

\begin{remark}
    In contrast, under uniform i.i.d.\ sampling without any adaptive structure, the probability of selecting a sample from a specific cluster $\mathcal{C}_k$ in a single draw is proportional to its size, i.e., $n_k / n$.  

For small clusters, this probability can be vanishingly small, making it unlikely to obtain even a single sample from such clusters within a block of fixed size.  
As a result, ensuring coverage across all clusters requires significantly more samples, leading to much weaker and less scalable sample complexity guarantees.
\end{remark}
\begin{remark}
It is easy to show that, with high probability, each block yields a unique sampled index. 
This justifies the assumption that accumulated coverage has independent noise.
\end{remark}
\end{proof}
\section{Proof of Theorem~\ref{thm:spectral}}

\makeatletter
\setcounter{theorem}{0} 
\renewcommand{\thetheorem}{2}  
\begin{theorem}[Spectral Approximation of $A_0$]
Let $\widehat{A} =  \widetilde{A}_S^\top \widetilde{A}_{S,S}^+ \widetilde{A}_S$ denote the estimator formed by low-rank reconstruction from  $m (\geq C(\delta) t K \log K)$ adaptively sampled rows of $\widetilde{A}$. Then,
\begin{align}
\| \widehat{A} - A_0 \|_{\rm op} \leq \frac{D n}{\sqrt{t} \cdot \min_k p_k},
\end{align}
with high probability, where $p_k$ is the cluster affinity for cluster $\mathcal{C}_k$, and $D := D(\lambda, K, \vec{p}) > 0$ depends on the parameter $\lambda$ of the noise matrix, the number of clusters $K$, and the affinities vector $\vec{p}$.
\end{theorem}
\addtocounter{theorem}{-1}  
\renewcommand{\thetheorem}{\arabic{section}.\arabic{theorem}} 
\makeatother

\begin{proof}
    To build intuition we first consider the case where $E =0$. 
    In this case
    \[
    \widehat A = A_{0}
    \]

    When there are $t$ samples available per row, under our block structure assumption we can say that 
    $\operatorname{rank}(A_S) =K$, the number of unique clusters.
    The following lemma allows us to then perfectly reconstruct $A_0$.
\begin{lemma}[\citet{owen2009bi}]
Let $\mathbf{M} \in \mathbb{R}^{n \times n}$ be a matrix, and suppose $\mathbf{M}_1 \in \mathbb{R}^{s \times n}$ consists of a subset of $s$ rows of $\mathbf{M}$ such that $\operatorname{rank}(\mathbf{M}_1) = \operatorname{rank}(\mathbf{M})$. Let $\mathbf{M}_{11} \in \mathbb{R}^{s \times s}$ denote the corresponding submatrix formed from $\mathbf{M}_1$. Then,
\[
\mathbf{M} = \mathbf{M}_1^\top \mathbf{M}_{11}^{+} \mathbf{M}_1.
\]
\end{lemma}
Take $M_1 = A_S$ and $M_{11} = A_{S,S}$. 
With high probability Theorem~\ref{thm:coverage} implies that $\operatorname{rank} (A_{S,S}) = \operatorname{rank}(A) = K.$
This
gives us the result that
$A_0 = \widetilde{A}_S^\top \widetilde{A}_{S,S}^+ \widetilde{A}_S$,
with high probability.

Now we move to the noisy case, when $E \neq 0$. 
In this case let us define the noiseless version of the $\widetilde{A}_S$ and $\widetilde A_{S,S}$ as $A_{S,S}$ and $A_{S}$

Define:
\begin{align*}
A_S &= A_0[S,:], \quad &
A_{S,S} &= A_0[S,S], \\
E_S &= \widetilde A_S - A_S, \quad &
E_{S,S} &= \widetilde A_{S,S} - A_{S,S}.
\end{align*}

Then $\widehat A = \widetilde A_S^\top \widetilde A_{S,S}^+ \widetilde A_S = (A_S + E_S)^\top (A_{S,S} + E_{S,S})^+ (A_S + E_S)$.

Let $M := A_{S,S} + E_{S,S}$. Expanding and regrouping,
\begin{align*}
\widehat A - A_0
&= A_S^\top (M^+ - A_{S,S}^+) A_S 
+ A_S^\top M^+ E_S 
+ E_S^\top M^+ A_S 
+ E_S^\top M^+ E_S \\
&= T_1 + T_2 + T_3 + T_4.
\end{align*}

Taking operator norms,
\begin{align*}
\| \widehat A - A_0 \| 
&\leq \|A_S\|^2 \cdot \|M^+ - A_{S,S}^+\| 
+ 2 \|A_S\| \cdot \|M^+\| \cdot \|E_S\| 
+ \|E_S\|^2 \cdot \|M^+\|.
\end{align*}

We now bound each term.
We first analyze the structure of the matrices $A_S$ and $A_{S,S}$ under the block-constant model.
Since $A_0$ is block-diagonal with $K$ clusters, and each block $A^{(k)} = p_k \cdot \mathbf{1}_{n_k \times n_k}$, any row $i$ of $A_S$ corresponding to cluster $\mathcal{C}_k$ has the form
\[
(A_S)_{i,j} =
\begin{cases}
p_k & \text{if } j \in \mathcal{C}_k, \\
0 & \text{otherwise}.
\end{cases}
\]
Hence, $\|(A_S)_{i,:}\|_2 = p_k \cdot \sqrt{n_k} \le p_{\max} \sqrt{n}$, since $n_k \le n$ and $p_k \le p_{\max}$, where $\|\|_{2}$ is the euclidean norm.

There are $t$ such rows per cluster for $K$ clusters, so $A_S \in \mathbb{R}^{tK \times n}$. 
Remember that the upper bound is tight (See Remark after proof of Theorem~\ref{thm:spectral}). 
Using the bound
\[
\|A_S\| \le \max_i \| (A_S)_{i,:} \|_2 \cdot \sqrt{tK} \le p_{\max} \sqrt{n} \cdot \sqrt{tK} = p_{\max} \sqrt{t n}.
\]

Next, consider $A_{S,S}$, the $tK \times tK$ submatrix of $A_0$ on rows $S$. 
$A_{S,S}$ consists of $K$ disjoint blocks, each of size $t \times t$ and equal to $p_k \cdot \mathbf{1}_{t \times t}$. 
Each such block has rank 1 and top eigenvalue $t p_k$.

Therefore,
\[
\sigma_{\min}(A_{S,S}) = t p_{\min}, \quad \sigma_{\max}(A_{S,S}) = t p_{\max}, \quad \|A_{S,S}^+\| = \frac{1}{t p_{\min}}.
\]

\textbf{Noise bounds.}  
We use the matrix Bernstein inequality \cite{bernstein2009matrix} for rectangular matrices with independent mean-zero sub-exponential entries.

\begin{lemma}[Matrix Bernstein Inequality, Rectangular Case {\cite{bernstein2009matrix}}]
Let $Z \in \mathbb{R}^{m \times n}$ be a random matrix with independent, mean-zero entries: $\mathbb{E}[Z_{ij}] = 0$ for all $i,j$. Assume each entry is sub-exponential, with
\[
\|Z_{ij}\|_{\psi_1} \le v \quad \text{for all } i,j,
\]
where $\|\cdot\|_{\psi_1}$ denotes the Orlicz norm associated with the sub-exponential function $\psi_1(x) = \exp(x) - 1$:
\[
\|X\|_{\psi_1} := \inf\left\{ t > 0 : \mathbb{E}[\exp(|X|/t) - 1] \le 1 \right\}.
\]
Define the row/column variance proxy
\[
\sigma^2 := \max \left\{ \max_i \sum_{j=1}^n \mathbb{E}[Z_{ij}^2], \;\; \max_j \sum_{i=1}^m \mathbb{E}[Z_{ij}^2] \right\}.
\]
Then there exists an absolute constant $c > 0$ such that for all $t > 0$, the operator norm of $Z$ satisfies
\[
\mathbb{P}\left( \|Z\| > t \right) \le (m + n) \exp\left( -c \min\left\{ \frac{t^2}{\sigma^2}, \frac{t}{v} \right\} \right).
\]
\end{lemma}

The entries of $E_S$ and $E_{S,S}$ are centered exponentials with $\psi_1$ norm $O(1/\lambda)$ and variance $O(1/\lambda^2)$. Therefore, with high probability,
\[
\|E_S\| \le \frac{C \sqrt{n}}{\lambda}, \quad
\|E_{S,S}\| \le \frac{C \sqrt{s}}{\lambda} = \frac{C \sqrt{t K}}{\lambda}.
\]

\textbf{Neumann series expansion.}  
If $\|E_{S,S}\| < \sigma_{\min}(A_{S,S})$, then $M = A_{S,S} + E_{S,S}$ is invertible and we may expand:
\[
M^{-1} = A_{S,S}^{-1} \sum_{k=0}^{\infty} \left( - E_{S,S} A_{S,S}^{-1} \right)^k,
\]
which converges when $\|E_{S,S} A_{S,S}^{-1}\| < 1$.

Since $\|A_{S,S}^{-1}\| = 1 / (t p_{\min})$, it suffices that
\[
\|E_{S,S}\| < \frac{1}{2} \cdot \frac{1}{\|A_{S,S}^{-1}\|} = \frac{t p_{\min}}{2},
\]
which holds if $t \ge C K / (\lambda^2 p_{\min}^2)$.

Hence,
\[
\|M^+\| \le \frac{2}{t p_{\min}}, \quad
\|M^+ - A_{S,S}^+\| \le \|A_{S,S}^+\|^2 \cdot \|E_{S,S}\| \cdot \sum_{k=0}^\infty \|E_{S,S} A_{S,S}^+\|^k
\le \frac{2 \|E_{S,S}\|}{(t p_{\min})^2}.
\]

Substituting the noise bound gives:
\[
\|M^+ - A_{S,S}^+\| \le \frac{2 C \sqrt{K}}{t^{3/2} p_{\min}^2 \lambda}.
\]

\textbf{Putting it all together.}
\begin{align*}
\|T_1\| &\le \|A_S\|^2 \cdot \|M^+ - A_{S,S}^+\|
\le p_{\max}^2 t n \cdot \frac{2 C \sqrt{K}}{t^{3/2} p_{\min}^2 \lambda}
= \frac{2 C p_{\max}^2 n \sqrt{K}}{p_{\min}^2 \lambda \sqrt{t}}, \\
\|T_2\| &= \|T_3\| \le \|A_S\| \cdot \|M^+\| \cdot \|E_S\| 
\le p_{\max} \sqrt{t n} \cdot \frac{2}{t p_{\min}} \cdot \frac{C \sqrt{n}}{\lambda}
= \frac{2 C p_{\max} n}{p_{\min} \lambda \sqrt{t}}, \\
\|T_4\| &\le \|E_S\|^2 \cdot \|M^+\| 
\le \left( \frac{C \sqrt{n}}{\lambda} \right)^2 \cdot \frac{2}{t p_{\min}}
= \frac{2 C^2 n}{\lambda^2 p_{\min} t}.
\end{align*}

The last term is lower order, so we absorb it into the constant.

Define:
\[
D(\lambda, K, \vec{p}) := \frac{2 C p_{\max}^2 \sqrt{K}}{\lambda} + \frac{4 C p_{\max}}{\lambda} + \frac{2 C^2 \sqrt{K}}{\lambda^2}.
\]

Then with high probability,
\[
\|\widehat A - A_0\| \le \frac{D(\lambda, K, \vec{p}) \cdot n}{\sqrt{t} \cdot p_{\min}}.
\]
\end{proof}

%% file: appendix_algo.tex
Algorithm~\ref{alg:balances_app} outlines an expanded BALANS procedure for batch correction. 
The algorithm takes as input a data matrix $X \in \mathbb{R}^{n \times d}$, associated batch labels $\vec b$, and three hyperparameters: 
the local neighborhood size $k$, 
a stopping threshold $\tau$ controlling convergence, 
and the block length $J$ for adaptive sampling.

At a high level, BALANS iteratively constructs a sparse affinity matrix by selectively sampling informative rows.
Sampling is performed in blocks of size $J$, using a temporary coverage vector $c_J$ to favor underrepresented regions of the data.
For each sampled index, BALANS computes an affinity row using a batch-conditional local scale kernel (see Algorithm~\ref{alg:bandwidth}). 
The sampling halts once $\tau$ consecutive blocks fail to introduce any new affinities.

After sufficient coverage is achieved, a low-rank affinity completion procedure (Algorithm~\ref{alg:completion}) reconstructs the full affinity matrix, which is then used to obtain the corrected representation $\widehat{X}$.

\begin{algorithm}[H]
\caption{BALANS Batch Correction}
\label{alg:balances_app}
\begin{algorithmic}[1]
\State \textbf{Input:} Data $X \in \mathbb{R}^{n \times d}$, batch labels $\vec b \in \{1,\dots,B\}^n$, nearest neighbor parameter $k$, stopping criteria $\tau$, reset interval $K$
\State \textbf{Output:} Corrected data matrix $\widehat{X} \in \mathbb{R}^{n \times d}$
\vspace{2mm}
\State Compute PCA on $X$ (optional)
\State Initialize cumulative coverage $c \gets \mathbf{0}$, sampled index set $S \gets \emptyset$
\State Initialize temporary coverage $c_J \gets \mathbf{0}$
\State Initialize $\texttt{no\_change\_count} \gets 0$
\While{$\texttt{no\_change\_count} < \tau$}
    \For{$j = 1$ to $J$}
        \State Sample index $i_t$ using Adaptive Sampling (Alg.~\ref{alg:adaptive-sampling})
        \State Compute sparse affinity row $A_{i_t,:}$ using Local Scale Kernel (Alg.~\ref{alg:bandwidth})
        \State $A_{\text{S}} \gets$ append $A_{i_t,:}$ as new row
        \State Let $\Delta \gets \text{number of indices where } A_{i_t,:} > 0 \text{ and } c_J = 0$
        \If{$\Delta = 0$}
            \State $\texttt{no\_change\_count} \gets \texttt{no\_change\_count} + 1$
        \Else
            \State $\texttt{no\_change\_count} \gets 0$
        \EndIf
        \State Update cumulative coverage: $c \gets c + A_{i_t,:}$
        \State Update temporary coverage: $c_J \gets c_J + A_{i_t,:}$
        \State $S \gets S \cup \{i_t\}$
    \EndFor
    \State Reset temporary coverage: $c_J \gets \mathbf{0}$
\EndWhile
\State Use Low-Rank Affinity Completion (Alg.~\ref{alg:completion}) to obtain corrected $\widehat{X} \in \mathbb{R}^{n \times d}$
\State \Return $\widehat{X}$
\end{algorithmic}
\end{algorithm}

\subsection{Batch-Dependent Local Scale computation}
Algorithm~\ref{alg:bandwidth} details the computation of a single sparse affinity row centered at point $x_i$, incorporating batch-dependent local scales. 
For each batch, a separate kernel bandwidth $\sigma_{i,b^\prime}$ is estimated using the distance to the $k$-th nearest neighbor in that batch. 
These batch-dependent scales are then assigned to each data point based on its batch label and used to compute affinities. 
The resulting affinity row is sparsified using elbow thresholding, where we detect the cutoff point beyond which affinities drop sharply. 
Specifically, we use the ruptures library with a sliding window change point detection algorithm: we sort the affinity values in decreasing order and apply a sliding window over this sequence. 
At each step, we compute the mean squared error of a two-segment piecewise constant fit (before and after the candidate change point). 
 At each step, we compute the total within-segment variance before and after the candidate change point and look for a drop that exceeds a quintile-based threshold (the 80th percentile of all such drops encountered during the scan)
This batch-aware design preserves local geometry while reducing inter-batch bias in affinity computation.

\begin{algorithm}[H]
\caption{Batch-Conditional Local Bandwidth Affinity for a Single Row}
\label{alg:bandwidth}
\begin{algorithmic}[1]
\State \textbf{Input:} Sampled row $x_i \in \mathbb{R}^d$, data matrix $X \in \mathbb{R}^{n \times d}$, batch labels $b$, bandwidth parameter $k$
\State \textbf{Output:} Sparse affinity row $A_{i,:} \in \mathbb{R}^n$
\vspace{2mm}
\State Compute distance vector: $d_j \gets \|x_i - x_j\|^2$ for all $j = 1$ to $n$
\State For each batch $b'$, compute $\sigma_{i,b'} \gets \|x_i - x_{b'}^{(k)}\|$ \Comment{$k$-th nearest neighbor in batch $b'$}
\State Broadcast $\sigma_{i,j} \gets \sigma_{i,b_j}$ for all $j$
\State Compute $A_{i,j} \gets \exp\left(-\frac{d_j}{\sigma_{i,j}^2}\right)$ for all $j$
\State Sparsify $A_{i,:}$ by elbow thresholding.
\State \Return $A_{i,:}$
\end{algorithmic}
\end{algorithm}

\subsection{Low Rank Smoothening Transform}
Algorithm~\ref{alg:completion} describes the low-rank completion step used to recover the full corrected data matrix $\widehat{X}$ from the sparsely sampled affinity matrix $A_S$. 
The method avoids constructing the full $n \times n$ affinity matrix, relying instead on efficient projections using the matrix $A_r$. 
The procedure first maps the data to a low-dimensional subspace via $A_r \cdot X$, then back-projects using $A_r^\top$ to obtain a smoothed reconstruction. 
A final renormalization ensures that the effective propagation matrix $P = A_r^\top A_r$ is row-stochastic. 

\begin{algorithm}[H]
\caption{Low-Rank Smoothening}
\label{alg:completion}
\begin{algorithmic}[1]
\State \textbf{Input:} Sparse affinity matrix $A_S \in \mathbb{R}^{s \times n}$, data matrix $X \in \mathbb{R}^{n \times d}$
\State \textbf{Output:} Completed data matrix $\widehat{X} \in \mathbb{R}^{n \times d}$
\vspace{1mm}

\State $r \gets$ row-sum vector of $A_S$ \Comment{$r_i = \sum_j A_S[i, j]$}
\State $D_r \gets \mathrm{diag}(1 ./ r)$ \Comment{Row normalization scaling}
\State $A_r \gets D_r \cdot A_S$ \Comment{Each row of $A_r$ sums to 1}

\State $Y \gets A_r \cdot X$ \Comment{Project data to low-rank subspace}

\State $c \gets A_r^\top \cdot \mathbf{1}_s$ \Comment{Column weights for renormalization}
\State $Z \gets A_r^\top \cdot Y$ \Comment{Back-project to full dimension}

\State $\widehat{X} \gets \mathrm{diag}(1 ./ c) \cdot Z$ \Comment{Normalize rows to make $P = A_r^\top A_r$ row-stochastic}

\State \Return $\widehat{X}$
\end{algorithmic}
\end{algorithm}

\subsubsection{Coverage-Based Adaptive Sampling}

Algorithm~\ref{alg:adaptive-sampling} selects  informative rows in each sampling step based on a coverage vector $c$. 
When uncovered points exist (i.e., entries with $c_j = 0$), the algorithm samples uniformly among them to ensure coverage expansion. 
Otherwise, it assigns sampling weights inversely proportional to coverage, favoring points that have been sampled less frequently.

\begin{algorithm}[H]
\caption{Coverage-Based Adaptive Sampling}
\label{alg:adaptive-sampling}
\begin{algorithmic}[1]
\State \textbf{Input:} Coverage vector $c \in \mathbb{R}^n$
\State \textbf{Output:} Index $i_t$ to sample at step $t$
\vspace{2mm}
\State $Z \gets \{j \in [n] : c_j = 0\}$ \Comment{Indices with zero coverage}
\If{$|Z| > 0$}
    \State Set $p_j = \frac{1}{|Z|}$ if $j \in Z$, else $p_j = 0$
\Else
    \For{each $j \in \{1, \dots, n\}$}
        \State Compute weight $w_j = \frac{1}{c_j}$
    \EndFor
    \State Normalize weights: $p_j = \frac{w_j}{\sum_k w_k}$
\EndIf
\State Sample $i_t \sim \text{Categorical}(p_1, \dots, p_n)$
\State \Return $i_t$
\end{algorithmic}
\end{algorithm}

%% file: run_time.tex
\makeatletter
\setcounter{theorem}{0} 
\renewcommand{\thetheorem}{3}  
\begin{theorem}[Computational Complexity of \textsc{BALANS}]
Let $n$ be the number of data points, $d$ the feature dimension, and $m = |S|$ the number of adaptively sampled rows after BALANS converges. 
Then the total computational complexity of \textsc{BALANS} is
\(
\mathcal{O}(n m (d + \log n)),
\)
where the first term accounts for sparse affinity construction and matrix multiplications, and the second term covers nearest neighbor search. 
When $m \ll n$, this is approximately linear in $n$.
\end{theorem}
\addtocounter{theorem}{-1}  
\renewcommand{\thetheorem}{\arabic{section}.\arabic{theorem}} 
\makeatother

\begin{proof}
Let $m = |S|$ be the number of indices selected before the stopping criterion is met. 
Let $k$ denote the local neighborhood size used in affinity computation (assumed constant in the complexity analysis). 
We examine each major step of Algorithms~\ref{alg:balances_app}--\ref{alg:completion}:

In each of the $m$ sampling steps, the algorithm performs the following operations:

\begin{itemize}
    \item \textit{Sampling an index:} Drawing from the adaptive distribution over $n$ elements can be done in $O(n)$ time.
    \item \textit{Computing distances:} Computing $\|x_i - x_j\|^2$ for all $j \in [n]$ requires $O(n d)$ time.
    \item \textit{Local bandwidth estimation:} Sorting distances within each batch to find the $k$-th neighbor takes $\mathcal{O}(n \log n)$ total (shared across batches).
    \item \textit{Affinity computation:} Applying the kernel and normalizing takes $\mathcal{O}(n)$ time, subsumed by sorting.
    \item \textit{Elbow sparsification:} Change point detection via a sliding window over the sorted affinities uses variance-based segmentation and early stopping, costing $\mathcal{O}(K)$ per row in practice and at worst $\mathcal{O}(n)$.
\end{itemize}

Thus, each iteration costs $\mathcal{O}(n(d + \log n))$, and $m$ such iterations cost
\[
\mathcal{O}(m n (d + \log n)).
\]

Given the sparse matrix $A_S \in \mathbb{R}^{m \times n}$, the following matrix operations are performed:

\begin{itemize}
\item \textit{Row normalization:} Computing the row sums and normalizing $A_S$ costs $\mathcal{O}(m n)$ in the worst case. However, since each row of $A_S$ contains at most $K$ nonzeros after sparsification, the practical cost is $\mathcal{O}(m K)$.
\item \textit{Projection:} Computing $Y = A_r X$ has worst-case cost $\mathcal{O}(m n d)$, but in practice only $\mathcal{O}(m K d)$ due to sparsity.
\item \textit{Back-projection:} Computing $Z = A_r^\top Y$ similarly costs $\mathcal{O}(m K d)$ in practice.
\item \textit{Final normalization:} Scaling $Z$ by $\mathrm{diag}(1./c)$ requires touching all entries in $Z$ and costs $\mathcal{O}(n d)$.
\end{itemize}

Thus, while the worst-case complexity is $\mathcal{O}(m n d)$, the practical runtime is typically much lower---approximately $\mathcal{O}(m K d)$---due to the controlled sparsity of the affinity rows.

\paragraph{3. PCA (Optional).}
If PCA is applied once to $X$ at the beginning to reduce dimensionality to $d_{\text{PCA}}$, it costs $\mathcal{O}(n d d_{\text{PCA}})$ via randomized SVD. 

\paragraph{4. Total Complexity.}
Combining all dominant terms:
\[
\mathcal{O}(m n (d + \log n)) + \mathcal{O}(m n d) = \mathcal{O}(m n (d + \log n)).
\]

Since $m \ll n$ due to the convergence criterion based on coverage, the overall cost is approximately linear in $n$.
\end{proof}

%% file: ablations.tex
We evaluate the sensitivity of \textsc{BALANS} to its three primary hyperparameters: (a) the nearest neighbor parameter $k$ used to compute local bandwidths, (b) the stopping criterion $\tau$, defined as the number of consecutive adaptive sampling rounds with no change in the sampled set, and (c) the block length $K$, which determines the number of samples drawn in each adaptive sampling block.

To assess robustness, we conduct experiments on both real-world datasets \textsc{JUMP 1} and \textsc{DEEP 1} and on synthetic datasets. 
Details of the real-world benchmarks are provided in Appendix~\ref{app:realworld}. 
For the synthetic setting, we use a structured Gaussian mixture model described in Appendix~\ref{app:synthetic}, consisting of 50{,}000 data points, 100 compounds, and 5 batches.

For each setting, we systematically vary one hyperparameter while holding the others fixed and report performance scores based on overall score. 

We find that \textsc{BALANS} achieves consistently strong performance across a wide range of hyperparameter settings. 
This robustness is observed in both real and synthetic scenarios, indicating that the method is robust to  parameter tuning within a wide range accross all hyperparameters.
This is summarised in Figure~\ref{fig:ablations}.

Based on these results, we recommend using the following default ranges in practice:
\[
k \in [5, 15], \qquad \tau \in [20, 100], \qquad K \in [20, 50].
\]

\textbf{Remark.} Unless otherwise stated, we use $k = 5$, $\tau = 50$, and $K = 50$ for all our experiments.

\begin{figure}[t]
    \centering
    \includegraphics[width=\linewidth]{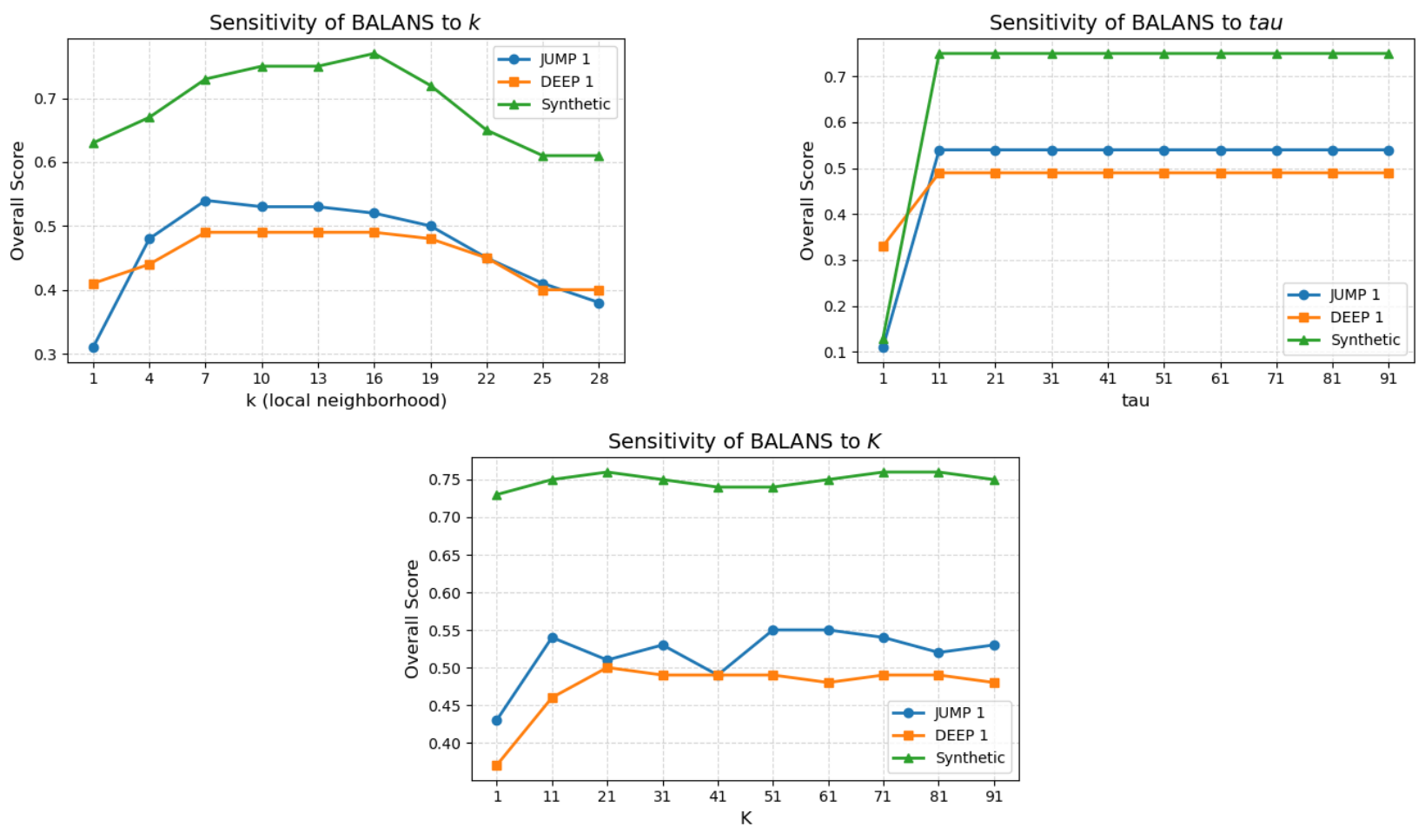}
    \caption{Sensitivity of \textsc{BALANS} to the hyperparameter $k$ across three datasets: JUMP 1, DEEP 1, and Synthetic. Each curve shows the overall score for varying values of $k$, demonstrating robustness across a wide range.}
    \label{fig:ablations}
\end{figure}


%% file: dataset.tex
All experiments were run using 4 NVIDIA GPUs (A100 or V100), with a maximum wall-clock time cutoff of 12 hours per method. 
For large-scale datasets, some baseline methods (CCA, RPCA) were unable to complete within this time limit, and are reported as timeouts where applicable.

We evaluate performance across a variety of real-world and synthetic datasets. 
Refer to Appendix~\ref{app:realworld} and Appendix~\ref{app:synthetic} for full dataset descriptions, including details on CellProfiler and DeepProfiler features, compound and batch parameters, and data dimensionality. 
All necessary preprocessing and postprocessing steps are described in the corresponding appendices.

\textsc{BALANS} was implemented in Python using efficient NumPy and SciPy routines. 
All affinity computations and matrix operations were performed using sparse matrix algebra with CSR (Compressed Sparse Row) format provided by `scipy.sparse`. 
Adaptive sampling was implemented with vectorized NumPy operations, and affinity sparsification leveraged nearest neighbor distances computed via `pd` . 
Unless otherwise stated, we use $k = 5$ (number of neighbors for local bandwidth estimation), $\tau = 50$ (early stopping threshold based on convergence of sampled rows), and $K = 50$ (block size for adaptive sampling). 
All experiments were run using a Snakemake pipeline for job orchestration and reproducibility, including automatic logging, checkpointing, and memory management.
We do not use PCA for our experiments, though we provide functionality in code to use it.
We did not see the use of PCA (when all dimensions are kept) change the results.

We compare \textsc{BALANS} to a wide range of competitive batch correction methods. 
Appendix~\ref{app:methods} provides complete implementation details, including hyperparameter configurations, early stopping criteria, and evaluation protocols used for each baseline.
For methods, we used either author-recommended default settings mostly from \cite{arevalo2024evaluating}.

All evaluation metrics are defined in Appendix~\ref{app:metrics}. 
These include metrics for biological signal preservation and batch correction.

%% file: evaluation_method.tex
\begin{table}[h]
\centering
\caption{Comparison of baseline and competing batch correction methods.}
\label{tab:method_comparison}
\begin{tabular}{l l l l}
\toprule
\textbf{Method} & \textbf{Category} & \textbf{Requires Batch Labels} & \textbf{Recomputation Needed} \\
\midrule
\textbf{BALANS} & Kernel + Sparse Sampling & Yes   & No  \\
Combat \cite{zhang2020combat}          & Statistical        & Yes   & No  \\
Sphering \cite{kessy2018optimal}       & Statistical        & No    & No  \\
CCA \cite{satija2015spatial}             & Linear Projection  & Yes   & No  \\
RPCA    \cite{satija2015spatial}         & Linear Projection  & No    & No  \\
FastMNN    \cite{haghverdi2018batch}     & Graph-based / KNN  & Yes   & Yes \\
Scanorama   \cite{hie2019efficient}    & Graph-based / KNN  & Yes   & Yes \\
Harmony \cite{korsunsky2019fast}        & Graph-based / KNN  & Yes   & Yes \\
\bottomrule
\end{tabular}
\end{table}

We conducted a comprehensive review of batch correction methods with a primary focus on those developed for single-cell RNA-seq (scRNA-seq), owing to the maturity and breadth of the literature in that domain. 
Based on a recent benchmarking analysis \citep{luecken2022benchmarking}, we selected a representative subset of top-performing methods that met the following criteria: 
(i) availability in Python or R, 
(ii) popularity according to citations, and (iii) coverage of diverse methodological approaches. 
These include linear models, mixture models, and nearest neighbor-based techniques. 
Table~\ref{tab:method_comparison} summarizes these methods.

\paragraph{Linear Methods}
\begin{itemize}
    \item \textbf{Combat}~\citep{zhang2020combat}: Models batch effects as additive and multiplicative noise, using a Bayesian framework to fit and correct for these effects via linear modeling.
    There are no hyperparameters that were tuned.
    \item \textbf{Sphering}~\citep{kessy2018optimal}: Applies a whitening transformation derived from negative control samples to the full dataset. 
    This method requires each batch to contain negative controls.
    A grid search was performed over the regularization parameter, and the best-performing value based on overall score was selected.

\end{itemize}

\paragraph{Mixture Model-Based Method}
\begin{itemize}
\item \textbf{Harmony}~\citep{korsunsky2019fast}: Alternates between clustering data into batch-diverse groups and correcting within clusters via an expectation-maximization (EM) framework. 
Key hyperparameters include the number of clusters and the diversity penalty parameter $\lambda$. 
We followed the hyperparameter choices used in \cite{arevalo2024evaluating}, which recommends default values.
\end{itemize}

\paragraph{Nearest Neighbor-Based Methods}
\begin{itemize}
    \item \textbf{fastMNN}~\citep{haghverdi2018batch}: A faster variant of MNN that performs PCA before nearest neighbor computation to scale better on large datasets.
    Key hyperparameters include the number of mutual nearest neighbors ($k$) and the number of principal components ($d$) retained during dimensionality reduction. 
    We follow the settings used in \cite{arevalo2024evaluating}.

    \item \textbf{Scanorama}~\citep{hie2019efficient}: Uses approximate nearest neighbor search in a low-dimensional space across all datasets, relaxing assumptions of shared subpopulations between batch pairs.
    Key hyperparameters include the number of neighbors ($k$) and the target dimensionality for integration. 
    Default values are provided by  \citet{arevalo2024evaluating}.

    \item \textbf{Seurat-CCA}~\citep{satija2015spatial}: Employs canonical correlation analysis to find a joint low-dimensional space, assuming substantial overlap in cell types across batches.
    Hyperparameters include the number of canonical correlation vectors ($\texttt{num.cc}$) and the number of integration anchors ($\texttt{k.anchor}$). 
    We used the values suggested in \cite{arevalo2024evaluating}.

    \item \textbf{Seurat-RPCA}~\citep{satija2015spatial}: Uses reciprocal PCA, allowing for greater heterogeneity across batches and offering better runtime on large datasets.
    The main hyperparameters are the number of principal components ($\texttt{n.pcs}$) and the number of integration anchors ($\texttt{k.anchor}$), both set to default values recommended by the authors.
\end{itemize}

These methods vary widely in their assumptions, computational requirements, and input dependencies. 
Most require explicit batch labels, except \textbf{Sphering}, which relies on negative controls. Four methods (\textbf{fastMNN}, \textbf{MNN}, \textbf{Harmony}, \textbf{Scanorama}) require recomputation of batch correction across the full dataset upon addition of new samples.
Table~\ref{tab:method_comparison}
 summarizes these methods.

%% file: metrics.tex
\begin{table}[h]
\centering
\caption{Summary of evaluation metrics used in our study.}
\label{tab:metrics_summary}
\begin{tabular}{l l l}
\toprule
\textbf{Metric} & \textbf{Description} & \textbf{Measures} \\
\midrule
Graph Connectivity \citep{luecken2022benchmarking}
    & Connectivity of same-batch cells post-correction & Batch \\
kBET \citep{buttner2019test}                  & Batch mixing in local neighborhoods (statistical test) & Batch \\
LISI (Batch) \citep{simpson1949measurement}           & Local batch entropy based on neighborhood diversity & Batch \\
Silhouette (Batch) \citep{rousseeuw1987silhouettes}    & Separation between batches in embedding space & Batch \\
LISI (Label)    \citep{simpson1949measurement}       & Local label entropy to assess biological clustering & Bio \\
Silhouette (Label) \citep{rousseeuw1987silhouettes}     & Separation of biological classes or labels & Bio \\
Leiden ARI   \citep{traag2019louvain}          & Agreement with ground-truth biological labels & Bio \\
Leiden NMI    \citep{traag2019louvain}         & Normalized mutual information with true labels & Bio \\
\bottomrule
\end{tabular}
\end{table}

To evaluate the performance of batch correction methods in image-based profiling, we adopt a suite of ten metrics, grouped into two categories: 
those measuring \textit{batch removal efficiency} and those assessing the \textit{preservation of biological variation}. 
These metrics are adapted from prior work in the single-cell RNA-seq domain \citep{luecken2022benchmarking}, and implemented using the \texttt{scib} package \citep{luecken2022benchmarking}.
Each metric is normalized to the $[0, 1]$ range, where 1 denotes ideal performance.
Table~\ref{tab:metrics_summary} summarizes these metrics.

\paragraph{Batch Removal Metrics}

\begin{itemize}

    \item \textbf{Silhouette Batch (ASW-Batch)} \citep{rousseeuw1987silhouettes}: 
    Computes the average silhouette width using batch labels as cluster assignments. 
    For each point $i$, let $a(i)$ be the average distance to other points in the same batch, and $b(i)$ the minimum average distance to points in a different batch. 
    Then the silhouette score $s(i)$ is defined as:
    \[
    s(i) = \frac{b(i) - a(i)}{\max\{a(i), b(i)\}}.
    \]
    The ASW-Batch score is the mean of $s(i)$ across all points. 
    Lower values indicate better batch mixing.
    But for our studies note that all scores are renormalised between $[0,1]$ such that higher values always indicate better performance

    \item \textbf{LISI Batch (bLISI)}\citep{simpson1949measurement}:
    The Local Inverse Simpson’s Index (LISI) quantifies the effective number of batches within a local neighborhood. 
    For each point $i$, let $\{p_{i,b}\}_{b=1}^B$ denote the proportion of neighbors belonging to batch $b$:
    \[
    \text{bLISI}(i) = \left( \sum_{b=1}^B p_{i,b}^2 \right)^{-1}.
    \]
    Higher bLISI values indicate more uniform batch mixing within neighborhoods.

    \item \textbf{kBET (k-nearest neighbor Batch Effect Test)} \citep{buttner2019test}:
    For each point, the null hypothesis is that batch labels in its $k$-nearest neighborhood are drawn from the global batch distribution. 
    A $\chi^2$ test is performed at each point and the fraction of tests that do not reject the null (i.e., accept) defines the kBET acceptance rate. Denote this as:
    \[
    \text{kBET} = \frac{1}{n} \sum_{i=1}^n \mathbf{1}\{\text{local distribution} \approx \text{global}\}.
    \]
    Higher acceptance indicates reduced batch effect.

    \item \textbf{Graph Connectivity} \citep{luecken2022benchmarking}:
    Constructs a $k$-nearest neighbor graph and evaluates the average connectivity within each compound. For compound $c$ with points $\mathcal{I}_c$, let $G$ be the adjacency matrix:
    \[
    \text{Conn}(c) = \frac{1}{|\mathcal{I}_c|^2} \sum_{i,j \in \mathcal{I}_c} G_{ij}, \quad \text{GC} = \frac{1}{C} \sum_{c=1}^C \text{Conn}(c).
    \]
    Higher values indicate that biologically similar cells remain connected post-correction.
    
    If all of them are close it means that they occur from diverse batches.

\end{itemize}

\paragraph{Biological Conservation Metrics}

\begin{itemize}

    \item \textbf{Silhouette Label (ASW-Label)} \cite{rousseeuw1987silhouettes}: 
    Same as ASW-Batch, but uses compound labels (e.g., perturbation identity) as the cluster assignment. High ASW-Label scores indicate well-separated biological clusters:
    \[
    \text{ASW-Label} = \frac{1}{n} \sum_{i=1}^n s(i), \quad s(i) = \frac{b(i) - a(i)}{\max\{a(i), b(i)\}}.
    \]

    \item \textbf{LISI Label (lLISI)}\citep{simpson1949measurement}: 
    Computes LISI using compound labels. 
    For each point $i$, let $\{p_{i,c}\}_{c=1}^C$ be the proportions of compound $c$ in its neighborhood:
    \[
    \text{lLISI}(i) = \left( \sum_{c=1}^C p_{i,c}^2 \right)^{-1}.
    \]
    Higher lLISI values indicate more concentrated biological neighborhoods.

    \item \textbf{Leiden ARI} \citep{traag2019louvain}:
    Applies the Leiden clustering algorithm to the corrected data to obtain cluster assignments $\hat{z}_i$, which are compared to ground truth compound labels $z_i$ using the Adjusted Rand Index:
    \[
    \text{ARI} = \text{AdjustedRandIndex}(\{\hat{z}_i\}, \{z_i\}).
    \]
    Higher ARI reflects better alignment with biological structure.
    We found that ARI is usually too small over all methods.

    \item \textbf{Leiden NMI} \citep{traag2019louvain}:
    Also uses Leiden clustering and computes the 
    Normalized Mutual Information (NMI) between predicted clusters and ground truth compound labels:
    \[
    \text{NMI} = \frac{2 I(\hat{Z}; Z)}{H(\hat{Z}) + H(Z)},
    \]
    where $I$ is mutual information and $H$ is entropy. 
    High NMI indicates strong biological preservation.

\end{itemize}

%% file: realworld.tex
We evaluate methods on eight major real-world datasets derived from high-throughput morphological profiling: the JUMP Cell Painting dataset \citep{chandrasekaran2023jump} and the Broad Bioimage Benchmark Collection (BBBC) \cite{ljosa2012annotated}. 

The \textbf{JUMP consortium} is a large-scale, industry-academic collaboration that spans multiple laboratories and experimental conditions, including diverse compound perturbations, batches, and imaging platforms. 
It comprises over 100,000 wells across 13 sources, making it a challenging testbed for robust batch correction. 
It has a cumulative $4$ datasets
CellProfiler \cite{carpenter2006cellprofiler} was used to extract profiles from the JUMP consortium

The \textbf{BBBC datasets} provides additional benchmarks under more controlled settings, including standard Cell Painting assays with annotated perturbation labels. 
It has a cumulative $3$ datasets.
These datasets are smaller in scale but offer cleaner ground truth and less inter-site variability, making them useful for evaluating biological signal preservation in less noisy regimes. 
DeepProfiler \citep{moshkov2024learning} was used to extract profiles from the BBBC datasets.

We split the full set of datasets into eight structured evaluation scenarios, combining data across multiple batches, sources, and imaging modalities. 
Each scenario is designed to test batch correction methods under different forms of heterogeneity, ranging from partner-specific confounding to microscope variability and feature extraction diversity. 
The scenarios include five settings derived from the JUMP dataset (JUMP 1–5) and three settings using deep learning-based feature extraction (DEEP 1–3). 
\textsc{BALANS} and baseline methods are evaluated on each of these using the full suite of biological conservation and batch correction metrics. 
Table~\ref{tab:scenarios} summarizes the composition of each scenario and whether it is reported in the main paper or supplement.

\subsection{CellProfiler}
\label{subsec:cellprofiler}

CellProfiler~\citep{carpenter2006cellprofiler} is an open-source software tool developed for the extraction of interpretable features from microscopy images. 
It is widely used in high-content screening and morphological profiling pipelines due to its modular design and ability to generate biologically meaningful measurements at the single-cell and well level.

In our experiments, we use CellProfiler to process Cell Painting assays across multiple datasets. 
The extracted feature set consists of approximately 4,762 features per well, aggregated across cells. These features span multiple biological and morphological categories, including:

\begin{itemize}
    \item \textbf{Shape}: Area, perimeter, eccentricity, and compactness of cellular components.
    \item \textbf{Intensity}: Mean, median, max, and standard deviation of pixel intensities across channels and compartments (nuclei, cytoplasm, and whole cell).
    \item \textbf{Texture}: Haralick features computed on grayscale co-occurrence matrices, describing textural patterns.
    \item \textbf{Granularity}: Quantifies the size and abundance of granule-like structures within cells, often linked to organelles or substructures.
    \item \textbf{Radial Distribution}: Measures the spatial distribution of intensity from the center of a cell outward.
    \item \textbf{Neighbors and Adjacency}: Statistics describing the spatial relationships between neighboring cells.
\end{itemize}

Feature extraction is performed channel-wise across five fluorescent stains: nuclei (Hoechst), nucleoli/cytoplasmic RNA, endoplasmic reticulum, Golgi/plasma membrane, and mitochondria. F
eatures are extracted from each stain and compartment (nucleus, cytoplasm, whole cell), then aggregated per well using the median across cells. 

The result is a biologically rich, high-dimensional vector representation for each well, capturing diverse morphological responses to perturbations.

\subsection{DeepProfiler}
\label{subsec:deepprofiler}

DeepProfiler~\citep{moshkov2024learning} is a deep learning framework for generating compact feature representations from high-content cellular microscopy images. 
Unlike CellProfiler, which relies on manually engineered features, DeepProfiler employs the EfficientNet architecture to learn feature embeddings directly from raw images, allowing it to capture complex, non-linear morphological patterns.

In this work, we use DeepProfiler as described in~\citep{moshkov2024learning}, which trains a deep network using weak supervision derived from experimental metadata (e.g., compound and concentration labels). 
The architecture used is a variant of ResNet-50, pretrained on ImageNet and fine-tuned using a classification head on perturbation labels.

\textbf{Input Processing.} Each five-channel Cell Painting image is processed as a stack, where each channel corresponds to one of the cellular compartments or stains. DeepProfiler either:
\begin{itemize}
    \item processes each channel independently (per-channel ResNet encoding),
    \item or combines them into a 5-channel input tensor for a unified CNN backbone.
\end{itemize}

\textbf{Training Objective.} The network is trained to distinguish between different treatments using a softmax cross-entropy loss, where each unique compound-concentration pair is treated as a separate class. 
Embeddings are taken from the penultimate layer of the network (before the classification head) and used as feature representations.

\textbf{Feature Output.} The output is a 384-dimensional feature vector.
Embeddings are aggregated across multiple fields of view and cells using mean pooling, resulting in a compact vector per well.

In our analysis, we evaluate DeepProfiler representations directly.

\subsection{DeepProfiler Architecture: Cell Painting CNN}
\label{subsec:deepprofiler_architecture}

The Cell Painting CNN model, integrated within the DeepProfiler framework, is designed to extract meaningful morphological features from high-content cellular images. 
This model employs the EfficientNet-B0 architecture, known for its balance between performance and computational efficiency.
EfficientNet-B0 comprises an initial convolutional layer followed by seven inverted residual blocks, incorporating MobileNetV2 structures with squeeze-and-excitation optimization, culminating in a final classification layer . 
It has approximately $2.4$ million parameters.

\textbf{Input Processing:} The model processes five-channel Cell Painting images, each channel corresponding to specific cellular components: DNA (nucleus), RNA (nucleoli and cytoplasmic RNA), endoplasmic reticulum (ER), Golgi apparatus and plasma membrane (AGP), and mitochondria. Each single-cell image is cropped to a size of $128 \times 128$ pixels, centered on the nucleus, preserving surrounding context without resizing. The channels are ordered as follows: DNA, RNA, ER, AGP, and Mitochondria .

\textbf{Training Strategy:} The model is trained using a weakly supervised learning approach, where each unique treatment (compound-concentration pair) serves as a distinct class. The training objective is to minimize the categorical cross-entropy loss, encouraging the model to learn features that distinguish between treatments. 
To address class imbalance, the training process involves sampling an equal number of cells per treatment per epoch, with oversampling for underrepresented treatments. 
Data augmentation techniques, including random cropping, horizontal flipping, rotation, and brightness/contrast adjustments, are applied to enhance generalization and prevent overfitting .

\textbf{Feature Extraction:} Post-training, features are extracted from the \texttt{block6a\_activation} layer of the EfficientNet-B0 model. 
This layer provides a 384-dimensional feature vector for each single-cell image, capturing nuanced morphological variations induced by different treatments. 
These features can be aggregated (e.g., via mean pooling) across multiple cells or fields of view to generate well-level profiles suitable for downstream analyses .

\textbf{Implementation Details:} The model is implemented using TensorFlow 2.5 and requires Python 3.6 or higher. 
GPU acceleration is recommended for efficient processing. 
Pretrained weights for the Cell Painting CNN are available and can be utilized for profiling new datasets without the need for retraining .

\subsection{JUMP Consortium Data}

The JUMP Cell Painting Consortium has assembled a large collection of morphological profiles generated using the Cell Painting assay. In this work, we primarily focus on datasets from the JUMP Cell Painting (JUMP-CP) project, particularly the \texttt{cpg0016-jump} dataset (abbreviated as \texttt{cpg0016}), which was produced during the core data generation phase. This dataset spans 13 distinct data-producing sites (or sources) and includes a diverse set of biological perturbations:
\begin{itemize}
    \item \textbf{Small-molecule compounds:} 116{,}753 perturbations.
    \item \textbf{ORF gene overexpression:} 15{,}142 perturbations.
    \item \textbf{CRISPR gene knockouts:} 7{,}977 perturbations.
\end{itemize}

To facilitate batch alignment and longitudinal consistency, a positive control plate termed \textbf{JUMP-Target-2-Compound} (referred to as \emph{Target2} plates) was included in every batch. This plate consists of 302 diverse compounds and enables alignment both within the dataset and with future datasets generated outside the consortium. The remainder of the plates, which contain the full set of 116{,}753 chemical perturbations, are referred to as \textbf{Production plates}.

Each Production plate includes a set of negative and positive controls to identify or correct for experimental artifacts. Specifically, wells treated with \textbf{Dimethyl Sulfoxide (DMSO)} serve as negative controls. These controls support detection and correction of plate-to-plate variation and serve as a morphological baseline for identifying perturbations with detectable effects.

\subsubsection{Preprocessing for JUMP}
\label{subsec:preprocessing}

The preprocessing pipeline largely follows \citet{arevalo2024evaluating}.
Each well in the JUMP Cell Painting dataset is represented by a 4,762-dimensional feature vector extracted via CellProfiler, capturing shape, intensity, texture, and pixel-level statistics. 

\paragraph{(1) Variation Filtering.}  
To remove uninformative features, we compute a robust plate-wise variation score based on control wells. For each feature $X$, let $X_e = \mathrm{median}(X)$ and $\sigma_e = \mathrm{median}(|X_i - X_e|)$. Define the coefficient of variation as:
\[
\mathrm{Cvar}(X) = \frac{\sigma_e}{X_e}.
\]
Features with $\mathrm{Cvar} < 10^{-3}$ on any plate are discarded.

\paragraph{(2) Median Absolute Deviation (MAD) Normalization.}  
We standardize feature values on each plate using control wells. Specifically, for each feature $X$ and well $i$, we apply the transformation:
\[
\bar{X}_i = \frac{X_i - X_e}{\sigma_e},
\]
where $X_e$ and $\sigma_e$ are defined as above and computed using only control wells.

\paragraph{(3) Rank-Based Inverse Normal Transformation (INT).}  
We apply a non-parametric normalization to each feature using the rank-based INT:
\[
Y_i = \Phi^{-1} \left( \frac{r_i - c}{N - 2c + 1} \right),
\]
where $r_i$ is the rank of sample $i$, $N$ is the number of samples, and $c = 3/8$ .
$\Phi^{-1}$ denotes the inverse cumulative distribution function of the standard normal distribution.

\paragraph{(4) Feature Selection.}  
To remove redundancy, we apply the \texttt{correlation\_threshold} function from \texttt{pycytominer} \cite{serrano2025reproducible}. Pairs of features with Pearson correlation above 0.9 are identified, and within each group, the feature with the highest total correlation to all others is excluded.

This pipeline ensures robust normalization, eliminates low-information and redundant features, and facilitates downstream analysis across heterogeneous batches and sources.

\subsection{BBBC Data}
\label{subsec:bbbc_data}

The Broad Bioimage Benchmark Collection (BBBC) provides publicly available microscopy datasets designed for evaluating image analysis algorithms. In the study "Learning representations for image-based profiling of perturbations"~\citep{moshkov2024learning}, three BBBC datasets were utilized to assess the effectiveness of the proposed DeepProfiler framework:

\begin{itemize}
    \item \textbf{BBBC037}: This dataset comprises images of U2OS cells subjected to overexpression of various genes. Specifically, it includes 302 unique genetic perturbations, each with multiple replicates. The dataset is instrumental in evaluating the model's ability to capture phenotypic variations induced by genetic modifications.

    \item \textbf{BBBC022}: This dataset contains images of U2OS cells treated with a diverse set of small molecules, encompassing over 1,600 chemical perturbations. It serves as a benchmark for assessing the model's capacity to discern morphological changes resulting from chemical treatments.

    \item \textbf{BBBC036}: This dataset features images of A549 cells exposed to a range of chemical compounds, totaling approximately 1,200 perturbations. It provides a platform to evaluate the generalizability of the learned representations across different cell lines and treatment types.
\end{itemize}

Each dataset includes five-channel Cell Painting images, capturing various cellular components such as the nucleus, endoplasmic reticulum, and mitochondria. The diversity in perturbation types and cell lines across these datasets enables a comprehensive evaluation of the DeepProfiler's representation learning capabilities.

\subsection{Scenarios}
\label{subsec:scenarios}

We define eight evaluation scenarios designed to capture a range of technical and biological variability across real-world and benchmark datasets. These include five scenarios derived from the JUMP Cell Painting Consortium (JUMP 1–5) and three from publicly available BBBC datasets (DEEP 1–3). Each JUMP scenario corresponds to a specific setup of labs, microscopes, and perturbation scales, while DEEP scenarios mirror the BBBC datasets used in the DeepProfiler study~\citep{moshkov2024learning}.

For clarity and comparability, we explicitly map these to Scenarios 1–5 from the benchmark study on batch correction methods~\citep{arevalo2024evaluating}, as follows:

\begin{itemize}
    \item \textbf{JUMP 3 $\leftrightarrow$ Scenario 1}: Data from a single lab using a single microscope, with 302 compound perturbations. Serves as a minimal-variability baseline.

    \item \textbf{JUMP 4 $\leftrightarrow$ Scenario 2}: Single lab using multiple microscopes. Contains 302 compounds. Introduces moderate technical variation due to imaging hardware.

    \item \textbf{JUMP 5 $\leftrightarrow$ Scenario 3}: Data from multiple labs, each using its own microscope. Contains over 80,000 chemical perturbations, reflecting large-scale cross-partner variability.

    \item \textbf{JUMP 1 $\leftrightarrow$ Scenario 4}: Three labs using a shared microscope model, with 302 compounds. Captures inter-lab variation under standardized equipment.

    \item \textbf{JUMP 2 $\leftrightarrow$ Scenario 5}: Five labs using the same microscope type, now scaling to over 80,000 compounds. Represents the most extensive batch diversity within controlled hardware.
    \item \textbf{DEEP 1 (BBBC036)}: Chemical perturbations in A549 cells, spanning approximately 1,200 compounds. Allows evaluation across a different cell line and biological context.
    
    \item \textbf{DEEP 2 (BBBC037)}: Genetic perturbation profiling in U2OS cells. Consists of 302 gene overexpression treatments, used to evaluate biological signal preservation under weak supervision.

    \item \textbf{DEEP 3 (BBBC022)}: Small molecule perturbations in U2OS cells. Includes over 1,600 chemical compounds with known annotations.

\end{itemize}

This structured set of scenarios spans single-lab to multi-lab conditions, microscope variability, and both engineered and learned feature spaces.
It provides a comprehensive framework for assessing batch correction performance under realistic and challenging conditions.

\begin{table}[t]
\centering
\caption{Summary of JUMP and DEEP evaluation scenarios. Each combines various sources, labs, and microscopes, with a range of perturbations and profiling methods.}
\label{tab:scenarios}
{\scriptsize
\begin{tabular}{lccccccc}
\toprule
\textbf{Scenario} & \textbf{\# Sources} & \textbf{\# Labs} & \textbf{\# Microscopes} & \textbf{\# Perturbations} & \textbf{\# Datapoints} & \textbf{Profiling Method} & \textbf{Main/Supp.} \\
\midrule
JUMP 1 & 3  & 3 & 3 & 302       & 30{,}000   & CellProfiler & Main \\
JUMP 2 & 5  & 5 & 3 & 80,000+   & 400{,}000  & CellProfiler & Main \\
JUMP 3 & 1  & 1 & 1 & 302       & 5{,}000    & CellProfiler & Supp. \\
JUMP 4 & 3  & 1 & 1 & 302       & 150{,}000  & CellProfiler & Supp. \\
JUMP 5 & 3  & 3 & 1 & 80,000+   & 300{,}000  & CellProfiler & Supp. \\
DEEP 1 & 1  & 1 & 1 & 1,000+    & 20{,}000   & DeepProfiler & Main \\
DEEP 2 & 1  & 1 & 1 & 300+      & 5{,}000    & DeepProfiler & Supp. \\
DEEP 3 & 1  & 1 & 1 & 1,500+      & 10{,}000    & DeepProfiler & Supp. \\
\bottomrule
\end{tabular}
}
\end{table}

\subsection{Full Results}

Here we provide tables summarizing the performance metrics for all experimental scenarios not included in the main paper text, specifically \textbf{JUMP 3–5} and \textbf{DEEP 2–3}.

\textsc{BALANS} consistently achieves strong performance across all settings. 
It performs best in high batch-diversity scenarios such as \textbf{JUMP 3-5}, where batch effects are more pronounced and challenging. 
This makes it particularly well-suited for applications like Cell Painting, where batches can vary significantly across perturbation types, plates, and imaging conditions.
 \textsc{BALANS} also maintains competitive accuracy in lower-diversity settings such as \textbf{DEEP 2–3}, demonstrating its robustness across different regimes of batch structure.

\textsc{BALANS} is broadly effective across real-world imaging datasets.

\begin{table}[ht]
\centering
\caption{
Evaluation scores across the real-world batch correction benchmark (\textbf{JUMP 3}).
\textsc{BALANS} performs best on batch metrics and achieves strong overall performance while preserving biological structure.
}
{\scriptsize
\setlength{\tabcolsep}{5pt}
\begin{tabular}{lccccccccccc}
\toprule
Method        & Conn. & KBET & LISI-batch & Silh-batch & LISI-label & ARI  & NMI  & Silh-label & Avg-batch & Avg-label & Avg-all \\
\midrule
BALANS        & 0.76 & \textbf{1.00} & \textbf{0.48} & 0.81 & \textbf{0.98} & \textbf{0.09} & \textbf{0.63} & \textbf{0.59} & \textbf{0.76} & \textbf{0.57} & \textbf{0.67} \\
Seurat CCA    & 0.80 & \textbf{1.00} & 0.46 & 0.77 & \textbf{0.98} & 0.07 & 0.60 & 0.52 & 0.76 & 0.54 & 0.65 \\
Combat        & 0.78 & 0.99 & 0.41 & \textbf{0.84} & \textbf{0.98} & 0.06 & 0.56 & 0.53 & 0.76 & 0.53 & 0.65 \\
Seurat RPCA   & \textbf{0.81} & 0.99 & 0.43 & 0.78 & \textbf{0.98} & 0.06 & 0.60 & 0.53 & 0.75 & 0.54 & 0.64 \\
fastMNN       & \textbf{0.81} & \textbf{1.00} & 0.44 & 0.76 & \textbf{0.98} & 0.05 & 0.59 & 0.55 & 0.75 & 0.54 & 0.64 \\
Baseline      & 0.77 & 0.99 & 0.39 & 0.85 & \textbf{0.98} & 0.06 & 0.55 & 0.52 & 0.75 & 0.53 & 0.64 \\
Harmony       & 0.79 & 0.99 & 0.40 & 0.74 & \textbf{0.98} & 0.06 & 0.60 & 0.55 & 0.73 & 0.55 & 0.64 \\
Sphering      & 0.74 & 0.99 & 0.37 & \textbf{0.87} & \textbf{0.98} & 0.06 & 0.54 & 0.52 & 0.74 & 0.53 & 0.64 \\
Scanorama     & 0.45 & 0.89 & 0.50 & 0.62 & \textbf{0.98} & \textbf{0.09} & 0.52 & 0.44 & 0.61 & 0.51 & 0.56 \\
\bottomrule
\end{tabular}
\label{tab:jump3}
}
\end{table}

\begin{table}[ht]
\centering
\caption{
Evaluation scores for the large, diverse batch setting \textbf{JUMP 4} using DeepProfiler features. 
\textsc{BALANS} performs competitively on batch correction and label structure while maintaining strong overall performance.
}

{\scriptsize
\setlength{\tabcolsep}{5pt}
\begin{tabular}{lccccccccccc}
\toprule
Method        & Conn. & KBET & LISI-batch & Silh-batch & LISI-label & ARI  & NMI  & Silh-label & Avg-batch & Avg-label & Avg-all \\
\midrule
Seurat CCA    & 0.64 & 0.66 & 0.55 & \textbf{0.86} & \textbf{0.98} & 0.04 & 0.45 & \textbf{0.49} & \textbf{0.68} & 0.49 & \textbf{0.59} \\
BALANS        & 0.59 & 0.54 & 0.54 & \textbf{0.86} & \textbf{0.98} & \textbf{0.05} & \textbf{0.50} & \textbf{0.54} & 0.63 & \textbf{0.52} & 0.57 \\
Seurat RPCA   & \textbf{0.65} & 0.57 & 0.46 & 0.84 & \textbf{0.98} & \textbf{0.05} & 0.45 & \textbf{0.49} & 0.63 & 0.49 & 0.56 \\
Harmony       & 0.64 & 0.62 & 0.55 & \textbf{0.86} & \textbf{0.98} & 0.04 & 0.42 & 0.48 & 0.67 & 0.48 & 0.58 \\
Scanorama     & 0.52 & \textbf{0.68} & \textbf{0.68} & 0.80 & \textbf{0.98} & 0.04 & 0.42 & 0.47 & 0.67 & 0.48 & 0.58 \\
fastMNN       & 0.56 & \textbf{0.70} & 0.52 & 0.80 & 0.97 & 0.04 & 0.39 & 0.45 & 0.64 & 0.46 & 0.55 \\
Combat        & 0.62 & 0.24 & 0.16 & 0.76 & \textbf{0.98} & 0.03 & 0.42 & 0.48 & 0.45 & 0.48 & 0.46 \\
Baseline      & 0.60 & 0.28 & 0.17 & 0.75 & \textbf{0.98} & 0.04 & 0.42 & 0.48 & 0.45 & 0.48 & 0.46 \\
Sphering      & 0.58 & 0.22 & 0.13 & 0.76 & \textbf{0.98} & 0.03 & 0.39 & 0.48 & 0.42 & 0.47 & 0.45 \\
\bottomrule
\end{tabular}
\label{tab:jump4}
}
\end{table}

\begin{table}[t]
\centering
\caption{
Evaluation scores for the large, diverse batch setting (\textbf{JUMP 5}). 
\textsc{BALANS} performs best across all batch and biological structure metrics in this challenging setting.
CCA/RPCA could not be run due the $12$ hr cut-off we enforce
}
{\scriptsize
\setlength{\tabcolsep}{5pt}
\begin{tabular}{lccccccccccc}
\toprule
Method    & Conn. & LISI-batch & Silh-batch & LISI-label & ARI  & NMI  & Silh-label & Avg-batch & Avg-label & Avg-all \\
\midrule
BALANS     & 0.33 & \textbf{0.48} & \textbf{0.91} & \textbf{1.00} & 0.01 & \textbf{0.46} & \textbf{0.32} & \textbf{0.57} & \textbf{0.33} & \textbf{0.45} \\
Scanorama  & \textbf{0.34} & 0.41 & 0.83 & \textbf{1.00} & 0.01 & 0.30 & 0.27 & 0.53 & 0.31 & 0.42 \\
fastMNN    & \textbf{0.34} & 0.45 & 0.79 & \textbf{1.00} & \textbf{0.02} & 0.26 & 0.22 & 0.53 & 0.30 & 0.41 \\
Harmony    & \textbf{0.34} & 0.25 & 0.84 & \textbf{1.00} & 0.01 & 0.27 & 0.32 & 0.48 & 0.32 & 0.40 \\
Sphering   & \textbf{0.34} & 0.00 & 0.84 & \textbf{1.00} & 0.00 & 0.22 & \textbf{0.38} & 0.39 & 0.31 & 0.35 \\
Combat     & \textbf{0.34} & 0.02 & 0.81 & \textbf{1.00} & 0.00 & 0.24 & 0.31 & 0.39 & 0.31 & 0.35 \\
Baseline   & \textbf{0.34} & 0.01 & 0.80 & \textbf{1.00} & 0.00 & 0.23 & 0.31 & 0.38 & 0.30 & 0.34 \\
\bottomrule
\end{tabular}
\label{tab:jump2}
}
\end{table}

\begin{table}[t]
\centering
\caption{
Evaluation scores for the low diversity batch setting (\textbf{DEEP 2}). 
\textsc{BALANS} performs best across all batch and biological structure metrics in this challenging setting.
CCA/RPCA could not be run due the $12$ hr cut-off we enforce
}
{\scriptsize
\setlength{\tabcolsep}{5.7pt}
\begin{tabular}{lcccccccccc}
\toprule
Method    & Conn. & LISI-batch & Silh-batch & LISI-label & ARI  & NMI  & Silh-label & Avg-batch & Avg-label & Avg-all \\
\midrule
BALANS     & 0.24          & \textbf{0.63} & 0.87          & \textbf{1.00} & \textbf{0.03} & 0.35        & 0.41          & \textbf{0.58} & \textbf{0.45}      & \textbf{0.51} \\
Sphering    & \textbf{0.29} & 0.44          & \textbf{0.89} & 0.98          & \textbf{0.03} & 0.31        & \textbf{0.46} & 0.54      & 0.44      & 0.49       \\
Seurat RPCA & 0.27          & 0.61          & 0.79          & 0.97          & 0.02 & 0.28        & 0.38          & 0.56      & 0.41      & 0.48       \\
Harmony     & 0.26          & 0.60          & 0.80          & 0.97          & 0.02 & 0.26        & 0.39          & 0.55      & 0.41      & 0.48       \\
Seurat CCA  & 0.27          & 0.58          & 0.78          & 0.96          & 0.02 & 0.29        & 0.38          & 0.54      & 0.41      & 0.48       \\
Baseline    & 0.26          & 0.56          & 0.80          & 0.97          & 0.02 & 0.27        & 0.39          & 0.54      & 0.41      & 0.47       \\
Combat      & 0.26          & 0.57          & 0.81          & 0.96          & 0.02 & 0.28        & 0.39          & 0.55      & 0.41      & 0.48       \\
fastMNN     & 0.25          & 0.59          & 0.77          & 0.97          & 0.02 & 0.26        & 0.37          & 0.54      & 0.40      & 0.47       \\
Scanorama   & 0.22          & 0.53          & 0.81          & 0.99          & 0.02 & \textbf{0.36} & 0.34          & 0.52      & 0.43 & 0.47       \\
\bottomrule
\end{tabular}
\label{tab:deep2}
}
\end{table}
\begin{table}[t]
\centering
\caption{
Evaluation scores for the low diversity batch setting (\textbf{DEEP 3}). 
\textsc{BALANS} performs best across all batch and biological structure metrics in this challenging setting.
CCA/RPCA could not be run due the $12$ hr cut-off we enforce
}
{\scriptsize
\setlength{\tabcolsep}{5.7pt}
\begin{tabular}{lcccccccccc}
\toprule
Method    & Conn. & LISI-batch & Silh-batch & LISI-label & ARI  & NMI  & Silh-label & Avg-batch & Avg-label & Avg-all \\
\midrule
BALANS     & 0.28          & \textbf{0.65} & 0.85          & \textbf{1.00} & \textbf{0.04} & 0.36        & 0.42          & \textbf{0.59} & \textbf{0.46}      & \textbf{0.53} \\
Sphering    & \textbf{0.30} & 0.40          & \textbf{0.88} & 0.98          & 0.03 & 0.30        & \textbf{0.46} & 0.53      & 0.44      & 0.49       \\
Seurat RPCA & 0.27          & 0.63          & 0.78          & 0.97          & 0.03 & 0.29        & 0.39          & 0.56      & 0.42      & 0.49       \\
Harmony     & 0.27          & 0.62          & 0.79          & 0.97          & 0.03 & 0.27        & 0.40          & 0.56      & 0.42      & 0.49       \\
Seurat CCA  & 0.28          & 0.60          & 0.78          & 0.97          & 0.03 & 0.30        & 0.39          & 0.55      & 0.42      & 0.49       \\
Baseline    & 0.27          & 0.58          & 0.79          & 0.97          & 0.03 & 0.28        & 0.40          & 0.55      & 0.42      & 0.48       \\
Combat      & 0.27          & 0.59          & 0.80          & 0.96          & 0.03 & 0.29        & 0.40          & 0.55      & 0.42      & 0.48       \\
fastMNN     & 0.26          & 0.61          & 0.76          & 0.97          & 0.03 & 0.27        & 0.38          & 0.54      & 0.41      & 0.48       \\
Scanorama   & 0.23          & 0.54          & 0.79          & \textbf{1.00} & 0.03 & \textbf{0.37} & 0.34          & 0.52      & 0.44 & 0.48       \\
\bottomrule
\end{tabular}
\label{tab:deep3}
}
\end{table}

\subsection{UMAP plots}
To qualitatively assess batch mixing and biological structure preservation, we present 2D UMAP embeddings for selected scenarios. 
As shown in the plots, \textsc{BALANS} effectively aligns batches while maintaining clear separation between biological labels.
\begin{figure}[ht]
    \centering
    \includegraphics[width=\textwidth]{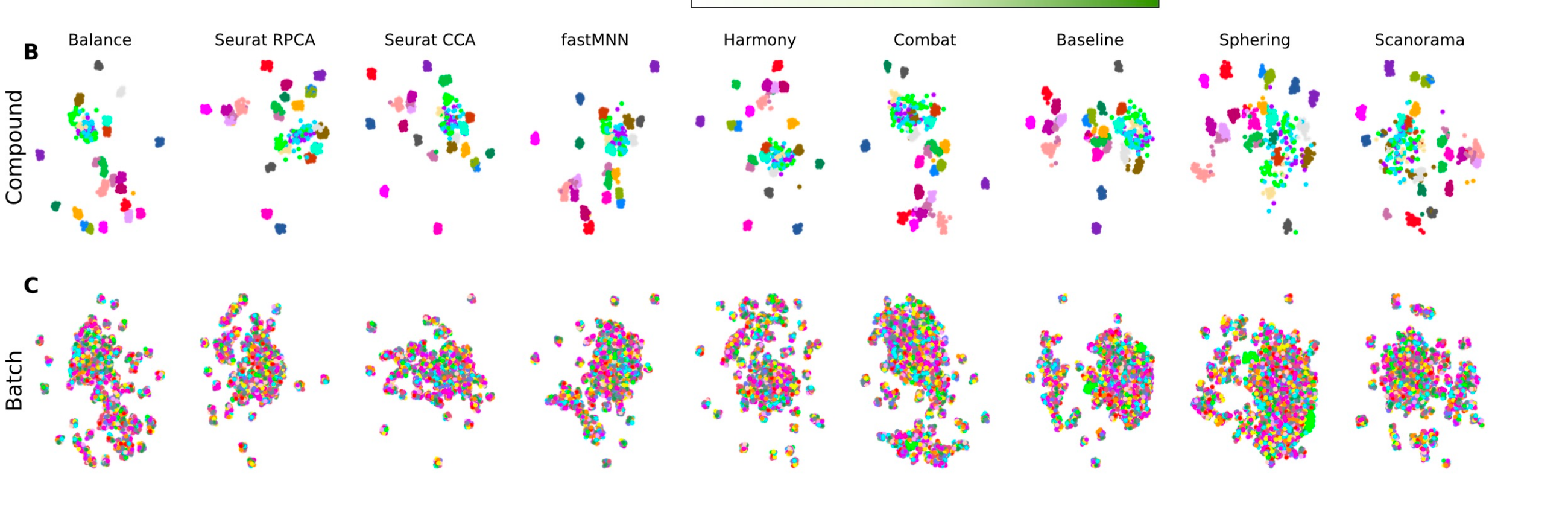}
    \caption{
        UMAP visualizations for JUMP 3. 
        (a) Colored by biological cluster, illustrating separation across a few representative compounds. 
        (b) Colored by batch. 
    }
    \label{fig:umap1}
\end{figure}
\begin{figure}[t]
    \centering
    \includegraphics[width=\textwidth]{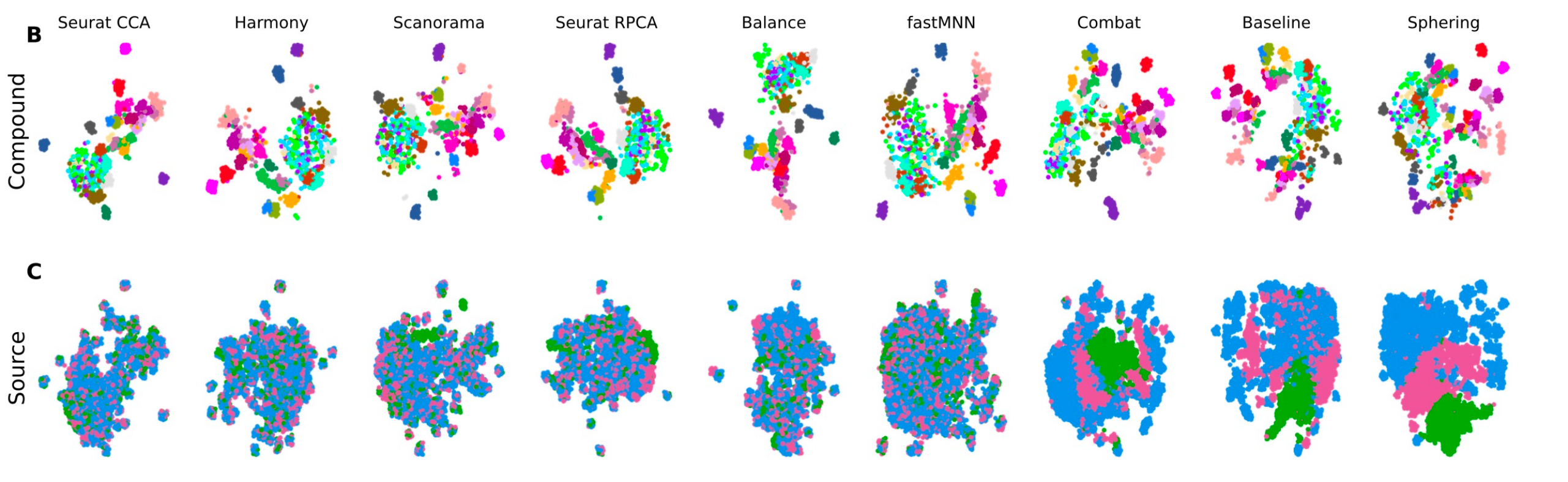}
    \caption{
        UMAP visualizations for JUMP 4. 
        (a) Colored by biological cluster, illustrating separation across a few representative compounds. 
        (b) Colored by source lab. .
    }
    \label{fig:umap2}
\end{figure}
\begin{figure}[t]
\vspace{-3mm}
    \centering
    \includegraphics[width=\textwidth]{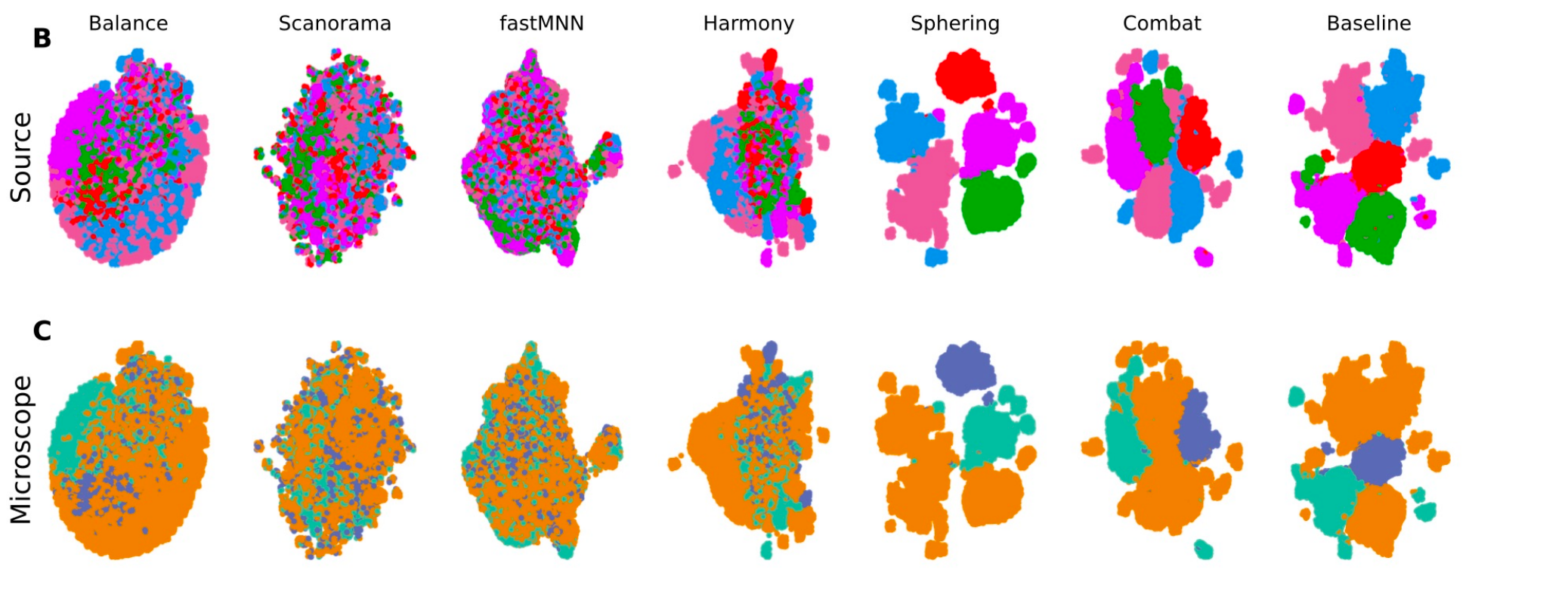}
    \caption{
        UMAP visualizations for JUMP 5. 
        (a) Colored by source lab. 
        (b) Colored by microscope. 
        There were too many compounds to plot UMAP colored by compound
    }
    \label{fig:umap3}
\end{figure}

\subsection{Comparison to BBKNN}
{

BBKNN is a widely used preprocessing tool that constructs a batch-aware $k$-nearest neighbor graph, making it useful for visualization and embedding methods such as UMAP. However, its functionality and objectives differ fundamentally from those of BALANS. Here we summarize the key distinctions and provide a fair comparison on the limited ground where the two methods overlap.

Although both BALANS and BBKNN rely on KNN-based constructions, their goals diverge sharply. BALANS is designed to produce corrected profiles by constructing and applying a batch-aware affinity operator to the data. BBKNN, in contrast, outputs only a graph and does not generate corrected embeddings or feature profiles. One could in principle attempt to multiply BBKNN's affinity matrix by the data to obtain corrected profiles, but the points below explain why this is not practically feasible.

BBKNN forms a full batch-balanced affinity matrix by merging per-batch KNN graphs, but this matrix cannot be efficiently applied to the data because it requires explicit operations with an $N \times N$ matrix. BALANS avoids this bottleneck. Instead of forming the full matrix, BALANS computes expressions of the form
\[
\hat{X} \approx (A_S^\top A_S) X,
\]
as described in lines~328--332, using a low-rank sampled operator $A_S$ that enables the affinity to be applied to the data without ever constructing or multiplying by the full matrix. The key difference is therefore not whether an affinity matrix can be defined (BBKNN clearly defines one), but whether it can be \emph{used} for scalable batch correction. BALANS is explicitly engineered for this; BBKNN is not.

Similarly, BBKNN's reported near-linear runtime applies only to its approximate KNN construction via Annoy. Since BBKNN never applies its affinity matrix to the data, this complexity does not extend to batch correction. Producing corrected profiles cannot be done in linear time, because even operating with a sparse affinity matrix requires more than $O(N)$ work. BALANS achieves linear-time affinity construction (see Supplementary~D) and then uses a low-rank approximation to make affinity--data multiplication scalable. In our experiments, using BBKNN's full affinity matrix to correct profiles is prohibitively slow on large datasets such as JUMP~2.

A further distinction lies in BALANS's use of batch-conditional local bandwidths, which support a principled elbow-detection thresholding mechanism (Supplementary~C.1). BBKNN does not compute batch-conditional scales; it identifies approximate neighbors per batch and retains a fixed number of them. This fixed, non-adaptive rule can perform poorly for unbalanced populations or heterogeneous densities. BALANS's batch-conditional scaling is therefore a unique and principled component that BBKNN does not provide.

Although BBKNN cannot produce corrected profiles, we nevertheless evaluate it on metrics computable from its KNN graph, such as batch-mixing scores, and, where feasible, attempt a surrogate affinity--data multiplication. BALANS consistently outperforms BBKNN on these graph-based metrics, and attempts to extend BBKNN to produce corrected profiles are computationally infeasible on large datasets.

\begin{table}[h]
\centering
\caption{Evaluation metrics for BBKNN and BALANS on JUMP-1 and JUMP-2. Numbers in parentheses indicate rankings.}
\renewcommand{\arraystretch}{1.2}
\begin{tabular}{lccccc}
\hline
\textbf{Dataset / Method} & \textbf{Graph conn.} & \textbf{LISI batch} & \textbf{LISI label} & \textbf{Leiden ARI} & \textbf{Leiden NMI} \\
\hline
\multicolumn{6}{c}{\textbf{JUMP-1 (Scenario 4)}} \\
\hline
BBKNN & 0.25 & 0.63  & 0.97 & 0.03  & 0.12 \\
BALANS & 0.54 & 0.44 & \textbf{0.98} & 0.05 & 0.46 \\
\hline
\multicolumn{6}{c}{\textbf{JUMP-2 (Scenario 5)}} \\
\hline
BBKNN & 0.33 & 0.006 & 1.000  & 0.000  & 0.229  \\
BALANS & 0.33 & \textbf{0.48} & \textbf{1.00} & 0.01 & \textbf{0.46} \\
\hline
\end{tabular}
\end{table} 

%% file: synthetic.tex
We generate synthetic datasets using a hierarchical Gaussian Mixture Model (GMM) that reflects label and batch structure in a controlled setting. The generative process proceeds as follows:

\begin{itemize}
    \item \textbf{Label-level generation:} For each semantic label $\ell \in \{1, \dots, L\}$, we sample a label-level mean $\mu_\ell \sim \mathcal{N}(0, \sigma^2_\text{label} \cdot I_d)$ in $d$ dimensions.
    \item \textbf{Batch-level generation:} For each batch $b \in \{1, \dots, B\}$ assigned to label $\ell$, we sample a batch-specific mean $\mu_{\ell,b} \sim \mathcal{N}(\mu_\ell, \sigma^2_\text{batch} \cdot I_d)$.
    \item \textbf{Point-level generation:} Each data point $x$ in batch $b$ with label $\ell$ is drawn as $x \sim \mathcal{N}(\mu_{\ell,b}, \sigma^2_\text{noise} \cdot I_d)$.
\end{itemize}

This setup allows us to control intra-label variability via $\sigma^2_\text{label}$, batch-induced perturbations via $\sigma^2_\text{batch}$, and overall noise via $\sigma^2_\text{noise}$.

\vspace{2mm}
We evaluate performance across the following many synthetic regimes, summarised in the figures.
BALANS is consitently faster compared to other methods.

\begin{figure}[ht]
    \centering
    \includegraphics[width=\textwidth]{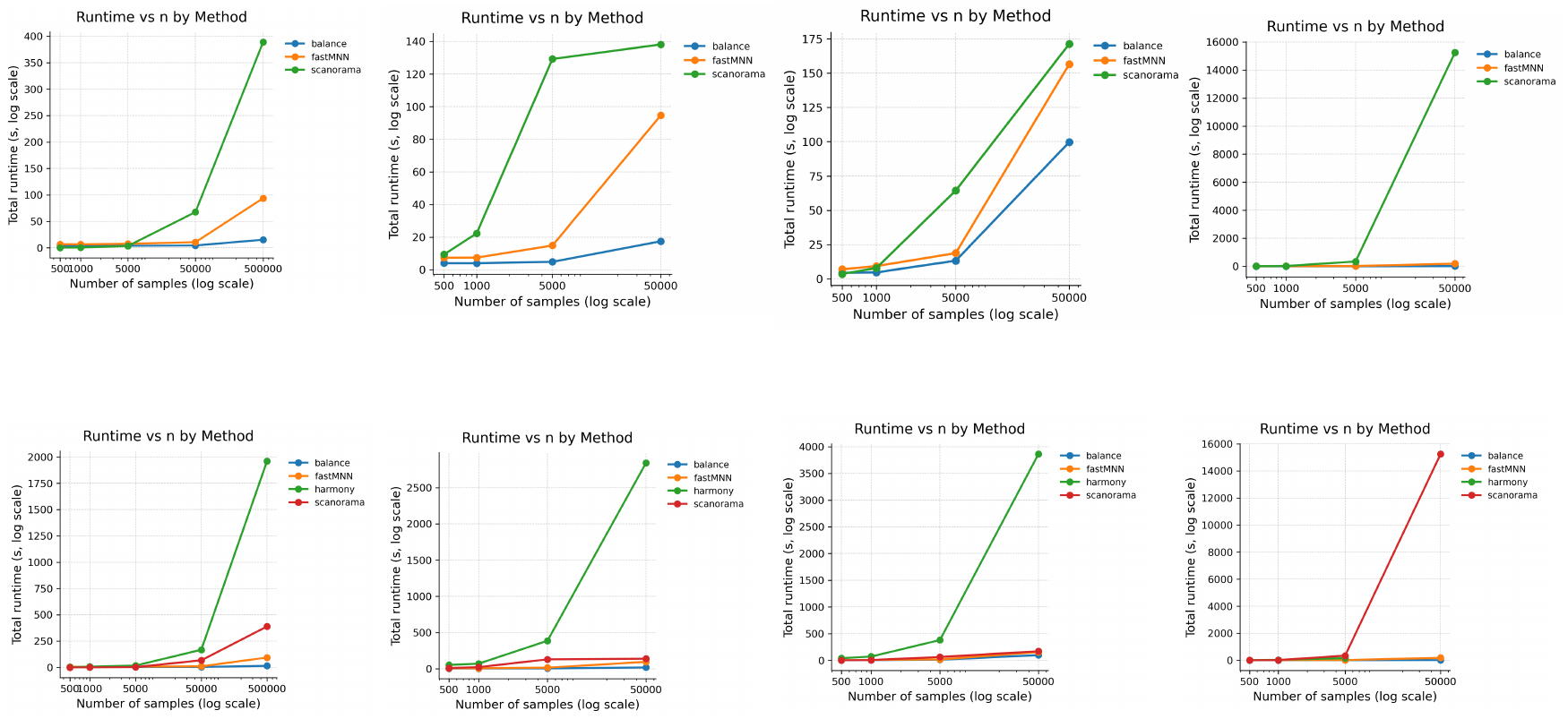}
    \caption{
    \textbf{Runtime comparisons across different synthetic scenarios.} 
    Each column presents two runtime bars: with and without Harmony, as Harmony is often significantly slower than other methods.
    From left to right, we consider:
    \textbf{(1)} A small-scale setup with 10 compounds, 10 batches, and 10 feature dimensions.
    \textbf{(2)} A medium-scale setup with 100 compounds, 5 batches, and 10 dimensions.
    \textbf{(3)} A robustness test varying batch-specific noise level $\sigma_{\text{batch}}$ under fixed structure.
    \textbf{(4)} A high-dimensional setting with 1000 features, 10 compounds, and 5 batches.
    Note that in high dimensions, Scanorama becomes notably slow compared to other methods.
    }
    \label{fig:runtime-synthetic}
\end{figure}

\begin{figure}[ht]
    \centering
    \includegraphics[width=0.8\textwidth]{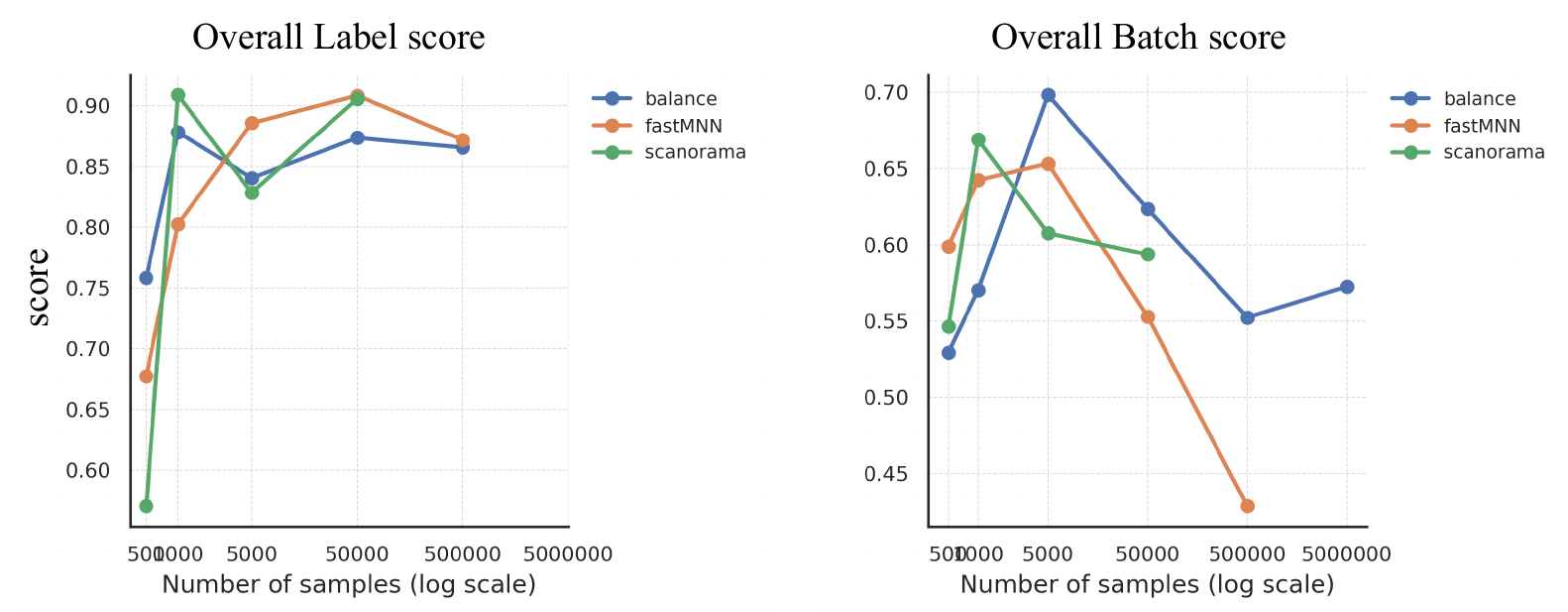}
    \caption{
    Accuracy scores for all methods in the runtime evaluation from the main paper. 
    The scores are fairly consistent across the fast-performing methods, suggesting stable performance despite computational constraints.
    \textsc{Harmony} score analysis did not complete because of an error, but we are working on fixing it.
    }
    \label{fig:runtime-accuracy}
\end{figure}

%% file: works.tex
\subsection{Deep Learning Representations for Cell Painting}

Recent years have seen an explosion of work applying deep learning to extract rich representations from Cell Painting images, moving beyond hand-engineered features. 
Efforts include \textsc{DeepProfiler} \citep{moshkov2024learning}, which introduced CNN-based feature extraction trained on perturbation classification. 
\citet{moshkov2024learning} leveraged transfer learning from pretrained networks such as ResNet, showing improved morphological profiling performance. 
More recently, \citet{wong2023deep} introduced contrastive pretraining strategies tailored to biological replicates, while \citet{doron2023unbiased} proposed a masked autoencoding approach that learns from noisy, unannotated images. 
\citet{kraus2024masked} further refined self-supervised learning pipelines by integrating structure-preserving loss terms and batch normalization during pretraining. 

\subsection{Batch Correction in Biological Data}

Batch correction is an extensively studied problem across domains including gene expression, proteomics, and image-based profiling. Early linear methods such as ComBat \citep{zhang2020combat} use empirical Bayes models to remove additive and multiplicative effects.
\citet{ando2017improving} propose advanced matrix factorization models for denoising high-throughput assays. \citet{vcuklina2021diagnostics} introduced adversarial learning to remove batch identity while preserving biological signal, and 
\citet{tran2020benchmark} provided a rigorous comparison of over 30 methods for scRNA-seq integration, emphasizing robustness to complex batch structures.
\citet{yu2023correcting} address the specific challenge of high-dimensional microscopy by integrating spatial and metadata information. 
Collectively, these methods span statistical, probabilistic, and representation learning approaches to batch correction.

\subsection{Local Bandwidth Estimation}

Local scale selection is foundational for adaptive affinity computation in manifold learning. Zelnik-Manor and Perona’s self-tuning kernel method \citep{zelnik2004self} adapts Gaussian bandwidths based on local $k$-nearest-neighbor distances, a widely used baseline. 
\citet{herrmann1997local} introduced scale-adaptive kernels for adaptive smoothing. 
\citet{knutsson1994local} proposed its use for quadrature wavelets.

\subsection{Low-Rank Approximation}

Low-rank approximations have long been a pillar of scalable learning. 
The Nyström method \citep{williams2000using} is a seminal technique that approximates large kernel matrices via landmark selection. 
\citet{drineas2005nystrom} extended this with randomized sampling strategies to improve accuracy. 
\citet{kumar2012sampling} proposed leverage-score-based sampling to reduce approximation error. 
\citet{musco2017recursive} introduced recursive Nyström approximations with provable spectral bounds, while 
\citet{gittens2016revisiting} offered a refined view to adapt this sampling for large scale datasets. PHATE is a visualization tool for nonlinear visualization of high-dimensional data, following the methodology introduced \citet{Moon2019PHATE}.

\subsection{Connection Between Theoretical Guarantees and Empirical Performance}
\label{app:theory_practice_connection}

In this section, we clarify the connection between the theoretical results developed in the main text and the empirical performance of BALANS. The algorithmic components used in practice follow directly from the theoretical guarantees, and the assumptions in our analysis are themselves grounded in empirical observations from biological data.

\paragraph{Adaptive sampling is derived directly from the theory.}
The theoretical results in Section~2.3 (Theorem~1) establish that the adaptive sampling scheme provides an efficient and near-optimal coverage of the underlying affinity matrix. This guarantee motivates the design of the sampling procedure used in BALANS: the empirical algorithm in Algorithm~1 (Section~3) implements precisely the adaptive strategy justified by the theory. Additional implementation details are provided in Supplementary Section~C.2.1. Thus, the sampling mechanism used in practice is not heuristic, but rather a direct instantiation of the theoretically optimal procedure.

\section{Theoretical guarantees justify the empirical approximation of the batch-free affinity matrix.}
The theory further shows that, under Assumption~2, the affinity matrix recovered by the BALANS procedure approximates the true batch-free affinity matrix with high probability. This result motivates the use of the recovered matrix for empirical batch correction: the corrected embeddings produced by BALANS are theoretically grounded approximations of the ideal, batch-free structure. The empirical evaluations demonstrate that this approximation leads to improved batch mixing while preserving biological signal.

\paragraph{The assumptions are motivated and validated by empirical observations.}
Importantly, Assumption~2 is not arbitrary but is directly informed by empirical observations on positive controls. As described in lines~187--191 of the main text, across multiple perturbations and plates we observe that within-cluster distances remain stable across batches, whereas batch effects manifest primarily as smooth extrinsic distortions. This empirical behavior motivates the modelling assumption used in the theory and supports its relevance for biological single-cell and image-based datasets.

\paragraph{Summary.}
The theoretical analysis informs the design of the BALANS algorithm, the algorithm implements the theoretically optimal components, and the assumptions underlying the theory are motivated and validated by empirical observations. Together, these connections ensure that the theoretical and empirical components of BALANS are tightly aligned.

\section{On the removal of the computation of an inverse}

In our framework, removing the Moore–Penrose pseudoinverse yields an approximation that is theoretically justified under certain idealized conditions—specifically when noise is negligible and when clusters exhibit relatively uniform sizes and affinity structure. 
In such settings, the affine operator is well-conditioned and the pseudoinverse contributes little beyond numerical refinement. We have added a short discussion in the Appendix outlining these conditions explicitly, as well as scenarios where the approximation may weaken, such as in the presence of strongly imbalanced clusters or substantial heterogeneity in affinities. This clarification makes explicit the regimes in which the simplified operator remains faithful to the original formulation.